%% file: main.tex
\documentclass{article}

\input{PackagesMacros}

\icmltitlerunning{\myTitle}

\begin{document}

\newtheorem{theorem}{Theorem}
\newtheorem{lemma}{Lemma}
\newtheorem{prop}{Proposition}
\newtheorem{claim}{Claim}
\newtheorem{observation}{Observation}
\newtheorem{definition}{Definition}

\twocolumn[
\icmltitle{\myTitle}

\begin{icmlauthorlist}
	\icmlauthor{Rotem Mulayoff}{technion}
	\icmlauthor{Tomer Michaeli}{technion}
\end{icmlauthorlist}

\icmlaffiliation{technion}{Department of Electrical Engineering, Technion -- Israel Institute of Technology, Haifa, Israel}

\icmlcorrespondingauthor{Rotem Mulayoff}{rotem.mulayof@gmail.com}

\icmlkeywords{Machine Learning, ICML, Flat Minima, Hessian}

\vskip 0.3in
]


\printAffiliationsAndNotice{}

\input{Abstract}

\input{Intro}

\input{Preliminaries}

\input{WarmUp}

\input{MainResults}

\input{Derivations}

\input{Experiments}

\input{RelatedWork}

\input{Conclusion}

\section*{Acknowledgments}
This research was supported in part by the Technion Ollendorff Minerva Center.

\prefixing
\bibliography{main}
\bibliographystyle{myStyle}
\nonprefixing

\clearpage
\icmltitlerunning{\myTitle: Supplementary Material}
\onecolumn
\icmltitle{\myTitle:\\Supplementary Material}
\appendix
\setcounter{equation}{0}
\setcounter{lemma}{0}
\setcounter{claim}{0}
\setcounter{definition}{0}
\setcounter{table}{0}

\renewcommand{\thelemma}{S\arabic{lemma}}
\renewcommand{\theclaim}{S\arabic{claim}}
\renewcommand{\thedefinition}{S\arabic{definition}}
\renewcommand{\thesection}{\Roman{section}}
\renewcommand{\theequation}{S\arabic{equation}}
\renewcommand{\thetable}{S\arabic{table}}

\input{Appendices}

\end{document}

%% file: PackagesMacros.tex
\usepackage{microtype}
\usepackage{amsmath}
\usepackage{amssymb}
\usepackage{amsthm}
\usepackage{xcolor}
\usepackage{enumitem}
\usepackage{graphicx}
\usepackage{caption}
\usepackage[labelformat=simple]{subcaption}
\graphicspath{{Figures/}}
\DeclareGraphicsExtensions{.pdf}

\usepackage[T1]{polski}
\usepackage[polish,english]{babel}
\usepackage{hyperref}
\usepackage[accepted]{myStyle}

\newcommand{\myTitle}{Unique Properties of Flat Minima in Deep Networks}
\newcommand{\R}{\mathbb{R}}

\newcommand{\vv}{ \boldsymbol{v} }
\newcommand{\uu}{ \boldsymbol{u} }

\newcommand{\rr}{ \boldsymbol{r} }
\newcommand{\zz}{ \boldsymbol{z} }
\newcommand{\qq}{ \boldsymbol{q} }
\newcommand{\XX}{ \boldsymbol{X} }
\newcommand{\optmlOper}{ \boldsymbol{T} }
\newcommand{\optmlOperScalar}{ \tau }
\newcommand{\UU}{ \boldsymbol{U} }
\newcommand{\bb}{ \boldsymbol{b} }
\newcommand{\BB}{ \boldsymbol{B} }
\newcommand{\RR}{ \boldsymbol{R} }
\newcommand{\VV}{ \boldsymbol{V} }
\newcommand{\MM}{ \boldsymbol{M} }

\newcommand{\bAA}{ \boldsymbol{A} }
\newcommand{\bAlpha}{\boldsymbol{\alpha}}
\newcommand{\Identity}{ \boldsymbol{I} }
\newcommand{\bPhi}{ \boldsymbol{\Phi} }
\newcommand{\SSS}{ \boldsymbol{S} }

\newcommand{\Hess}{ \boldsymbol{H} }
\newcommand{\HessTag}{ \hat{\boldsymbol{H}} }
\newcommand{\weights}{\boldsymbol{W}}
\newcommand{\wideminset}{\Omega_0}
\newcommand{\minset}{\Omega}
\newcommand{\weightsPointVec}{ {\boldsymbol{w}} }
\newcommand{\zeroVec}{ \boldsymbol{0} }
\newcommand{\loss}{\ell}
\newcommand{\argmin}{\text{argmin}}
\newcommand{\vectorize}[1]{ \text{vec}\left( #1 \right) }
\newcommand{\E}[1]{ \hat{ \mathbb{E} } \left[ #1 \right] }
\newcommand{\Ei}[1]{\mathrm{E} [ #1 ] }
\newcommand{\EmpCovMat}[1]{ \hat{ \boldsymbol{\Sigma}}_{#1} }

\newcommand{\norm}[1]{\left\lVert #1 \right\rVert}
\newcommand{\normi}[1]{\| #1 \|}
\newcommand{\ie}{\emph{i.e.} }
\newcommand{\eg}{\emph{e.g.} }
\newcommand{\bPsi}{ \boldsymbol{\Psi}}
\newcommand{\scndPrtDev}[3]{\frac{\partial^2}{ \partial {#1}_{#2} \partial {#1}_{#3}}}

\makeatletter
\newcommand*{\addFileDependency}[1]{
	\typeout{(#1)}
	\@addtofilelist{#1}
	\IfFileExists{#1}{\typeout{File #1 O.K.}}{\typeout{No file #1.}}
}
\makeatother


%% file: Abstract.tex
\begin{abstract}
	It is well known that (stochastic) gradient descent has an implicit bias towards flat minima. In deep neural network training, this mechanism serves to screen out minima. However, the precise effect that this has on the trained network is not yet fully understood. In this paper, we characterize the flat minima in linear neural networks trained with a quadratic loss. First, we show that linear ResNets with zero initialization necessarily converge to the flattest of all minima. We then prove that these minima correspond to nearly balanced networks whereby the gain from the input to any intermediate representation does not change drastically from one layer to the next. Finally, we show that consecutive layers in flat minima solutions are coupled. That is, one of the left singular vectors of each weight matrix, equals one of the right singular vectors of the next matrix. This forms a distinct path from input to output, that, as we show, is dedicated to the signal that experiences the largest gain end-to-end. Experiments indicate that these properties are  characteristic of both linear and nonlinear models trained in practice.
\end{abstract}

%% file: Intro.tex
\section{Introduction}\label{sec:Intro}
Optimization methods can have implicit biases towards certain solutions \cite{strand1974theory,morgan1990generalization,neyshabur2014search}. In the context of deep network training, such biases have been shown to play key roles in shaping the properties of the learned model. For example, in binary classification of linearly separable data, among all linear separators that achieve the global minimum of the training loss, gradient descent (GD) converges to the maximum margin separator. This is true for shallow networks \cite{soudry2018implicit}, as well as for deep linear fully-connected models \cite{gunasekar2018implicit} and deep nonlinear networks with homogeneous activation functions  \cite{lyu2020gradient}. Implicit biases have been studied in many different context, including for linear convolutional networks \cite{gunasekar2018implicit}, matrix factorization \cite{gunasekar2017implicit}, weight normalization \cite{wu2019implicit}, and with different loss functions \cite{gunasekar2018characterizing}.

Perhaps the simplest mechanism through which GD and stochastic GD (SGD) can screen out solutions, is their inability to stably converge to sharp minima \cite{jastrzkebski2017three,wu2018sgd,simsekli2019tail}. In fact, in some cases, GD can only converge to the flattest of all minima (see Section \ref{sec:Preliminaries}). However, interestingly, the effect that this has on the resulting trained model, is not yet fully understood. \citet{keskar2016large} suggested that flat minima tend to generalize better. This was somewhat supported by follow up works, showing that in SGD, larger step sizes and smaller batch sizes impose convergence to flatter minima, which indeed generalize better empirically  \cite{jastrzkebski2017three,hoffer2017train,masters2018revisiting,smith2017bayesian}. However, \citet{dinh2017sharp} showed that for networks with ReLU activations, a re-parametrization of the weights can make any minimum arbitrarily sharper (without affecting generalization). This suggests that minimum sharpness is not directly related to generalization, thus begging the question: What \emph{does} the sharpness of the minimum affect?

Our goal in this paper is to unveil the properties of flat minima in deep neural networks. We specifically focus on linear models trained with a quadratic loss and define the sharpness of a minimum to be its maximal Hessian eigenvalue, which is the factor affecting stable convergence of GD and SGD. We start by showing that all minima become sharper as the network gets deeper. We discuss and illustrate the implications this has on the training process. 
We then move on to study the flattest minimum solutions. We prove that these networks possess a special structure, whereby the gain from the input to any intermediate layer is well behaved. Furthermore, consecutive layers in those solutions are coupled, forming a distinct path from input to output, which is dedicated to the signal that experiences the largest gain end-to-end. Interestingly, similar properties were recently shown to arise in deep linear networks for binary classification \cite{ji2019gradient}. However, in our case of vector-valued regression, the behaviors turn out to be more complex. We empirically illustrate that the properties we predict are also characteristic of nonlinear networks trained in practice.

%% file: Preliminaries.tex
\section{Problem Setting and Motivation}\label{sec:Preliminaries}
Consider an \(m\)-layer linear network whose $j$th layer performs multiplication by \( \weights_j \in \R^{d_{j} \times d_{j-1}} \). The end-to-end function \( f_\weightsPointVec: \R^{d_x} \mapsto \R^{d_y} \) implemented by this network is
\begin{equation}\label{eq:fDef}
	{f}_\weightsPointVec(x) = \weights_m \weights_{m-1} \cdots \weights_1 x,
\end{equation}
where we denoted $ \weightsPointVec = \vectorize{[\weights_1,\weights_2,\ldots,\weights_m]} \in \R^{N} $. Here, \( N = \sum_{j=1}^{m} d_{j} \times d_{j-1}  \) and we use the convention that $d_0=d_x$ and $d_{m}=d_y$. To ensure that the network can implement any linear function from \( \R^{d_x} \) to \( \R^{d_y} \), we assume that the dimensions of the internal representations are not smaller than those of the input or output, namely \( \min_j \{d_j\} \geq \min \{d_x,d_y\} \).

We focus on the quadratic training loss
\begin{equation}\label{eq:VectorLossFuncDef}
	\loss(\weightsPointVec) = \E{\big\| y-f_\weightsPointVec(x) \big\|^2},
\end{equation}
where \( \hat{ \mathbb{E} } \) denotes empirical mean over paired examples \( \{ (x_i,y_i) \}_{i = 1}^n \).  Note that if the input $ x $ lies in a low dimensional subspace, (\eg if the number of training examples $ n $ is smaller than the ambient dimension $ d_x $), then there exist directions $\tilde{\weightsPointVec}$ in parameter space such that $\loss(\weightsPointVec)=\loss(\weightsPointVec+\alpha \tilde{\weightsPointVec})$ for every $\weightsPointVec$ and every $\alpha\in\R$. Minima that differ along these directions may correspond to different end-to-end functions, yet they have the exact same loss landscape around them. This implies that the sharpness of a minimum is indifferent to the end-to-end function in our setting, and in particular it is not associated with generalization. In our scenario, the sharpness criterion is only sensitive to \emph{different implementations} of the same end-to-end function.

In light of this understanding, we assume that the empirical second-order moment matrix of \( x \), denoted by \(\EmpCovMat{x}\), is full rank. In this case, the end-to-end function minimizing the loss is unique and can be written as ${f}_{\weightsPointVec^*}(x)=\optmlOper x$, where
\begin{equation}\label{eq:OptimalLinearTransform}
	\optmlOper = \EmpCovMat{yx} \EmpCovMat{x}^{-1}
\end{equation}
with \( \EmpCovMat{yx} \) denoting the empirical cross second-order moment between \( y \) and \( x \). Thus, the set of global minima of \( \loss(\weightsPointVec) \) is
\begin{equation}\label{eq:MinimaSet}
	\minset = \left\{ \weightsPointVec \in \R^{N} \ : \ \weights_m \weights_{m-1} \cdots \weights_1 =  \optmlOper \right\}.
\end{equation}

Among all minima in $\minset$, GD and SGD can only stably converge to the flat ones (see App.~\ref{appendix:Stability}). Specifically, denote by \( \Hess_{\weightsPointVec} \) the Hessian matrix of \( \loss(\weightsPointVec) \) at $\weightsPointVec$ and define the sharpness of a minimum point \( \weightsPointVec^* \) to be \(\lambda_{\max} \big( \Hess_{\weightsPointVec^*} \big) \). Then $\weightsPointVec^*$ is not stable for GD and SGD if its sharpness is larger than \( 2/\eta \), where $\eta$ is the step-size \cite{wu2018sgd}. In other words, the larger the step size, the smaller the set of minima that are accessible by the optimizer. Particularly, when using the largest step size allowing convergence, we can only reach elements in the set of \emph{flattest} global minima,
\begin{equation}\label{eq:WidestMinimaSet}
	\wideminset = \underset{\weightsPointVec  \in  \minset}{\argmin} \   \lambda_{\max} \big( \Hess_{\weightsPointVec} \big).
\end{equation}

Our goal in this paper is to unveil the properties of solutions in $\wideminset$. Our motivation for doing so goes beyond large step-size training. Indeed, in many cases convergence to a point in $\wideminset$ is guaranteed also with a small step-size. For example, we have the following result for gradient flow (GD with an infinitesimal step size) and for GD with a small step size (see proof in App.~\ref{appendix:WidestMinimaConvergenceProof}).
\begin{lemma}\label{lemma:WidestMinimaConvergence}
	Assume that $ \EmpCovMat{x} = \Identity $, $d_y=d_x$, and that the weight matrices are all square and initialized to $\Identity $. Then:
	\begin{enumerate}[label=\roman*.]
		\item Gradient flow can only converge to a flattest minimum. \label{lem:convrgGF}
		
		\item If $ \optmlOper $ is positive definite and its top singular value is $\sigma_{\max}(\optmlOper)$, then GD with step size $\eta\leq\frac{1}{ 2m}\min\{1,(\sigma_{\max}(\optmlOper))^{-2(1-\frac{1}{m})}\}$ necessarily converges to a flattest minimum at a linear rate. \label{lem:convrgGD}
	\end{enumerate}
\end{lemma}
Note the relevance of this lemma to the practice of zero initialization for residual networks (ResNets) \cite{zhang2018fixup}. Indeed, linear networks with identity initialization can be viewed as linear ResNets with zero initialization.

%% file: WarmUp.tex
\section{Warm-Up: Scalar Networks}\label{sec:WarmUp}
Before we present our main results, it is insightful to examine the simple case where the input, output and all intermediate representations, are scalars. In this case, the end-to-end function \( f_\weightsPointVec(x) \) is given by
\begin{equation}
	f_\weightsPointVec(x) = \prod_{j=1}^{m} w_j x,
\end{equation}
where $ \weightsPointVec = [w_1,w_2,\ldots,w_m]^T \in \R^{m} $, and the quadratic loss is minimized when ${f}_{\weightsPointVec^*}(x)=\optmlOperScalar x$, with \( \optmlOperScalar = {\hat{\sigma}_{xy}}/{\hat{\sigma}_x^2} \). Thus, the set of global minima is given by 
\begin{equation}
\minset = \Big\{ \weightsPointVec \in \R^{m} \ : \ \prod_{j=1}^{m} w_j = \optmlOperScalar \Big\}.
\end{equation}
Observe that these global minima lie within connected valleys. For example, in the case of two layers, $\minset$ corresponds to the hyperbola $w_2=\tau/w_1$, shown in Fig.~\ref{fig:2D_example}. Parts of these valleys are sharper than others, and as the theory predicts, GD indeed does not converge to a narrow part of the valley, even when initialized nearby such a global minimum.

\begin{figure}[t!]
	\begin{subfigure}[t]{1.5in}
		\includegraphics[width=1.5in,trim={0 0 0 0},clip]{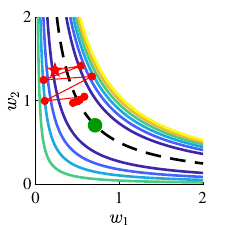}
		\caption{\centering\label{fig:2D_example_non_moment}GD without momentum}
	\end{subfigure}\begin{subfigure}[t]{1.5in}
		\includegraphics[width=1.5in,trim={0 0 0 0},clip]{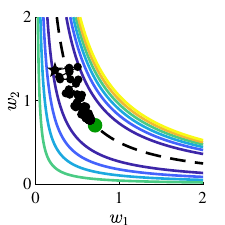}
		\caption{\centering\label{fig:2D_example_moment}GD with momentum}
	\end{subfigure}\begin{subfigure}[t]{0.5in}
		\includegraphics[width=0.7in,trim={1.3in 0in 0 0.15in},clip]{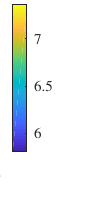}
	\end{subfigure}
	\caption{Level sets of the loss for a two-layer scalar network.	The dashed line corresponds to the set of global minima \( \minset \), and the green dot to the set of flattest minima \( \wideminset \). When GD is initialized nearby a sharp minimum (star), it does not converge to that minimum, and rather traverses the valley of minima until reaching a flat enough point. This occurs both with and without momentum.}
	\label{fig:2D_example}
\end{figure}

Direct computation (see App.~\ref{appendix:ScalarCaseWidestMinimum}) shows that for $\weightsPointVec\in\minset$,
\begin{equation}
\frac{\partial^2 \loss(\weightsPointVec)}{\partial w_q \partial w_k}
=\frac{ 2 \hat{\sigma}_x^2  \optmlOperScalar^2 }{w_k w_q}.
\end{equation}
Therefore, letting \( \zz = [w_1^{-1},w_2^{-1},\ldots,w_m^{-1}]^T \), we can express the Hessian matrix at a global minimum as
\begin{equation}
\Hess_\weightsPointVec = 2 \hat{\sigma}_x^2 \optmlOperScalar^2 \zz \zz^T.
\end{equation}
Evidently, the Hessian for scalar networks is a rank-one matrix whose (single) nonzero eigenvalue is 
\begin{equation}\label{eq:LambdaMaxScalar}
\lambda_{\max}\big( \Hess_\weightsPointVec \big) = 2 \hat{\sigma}_x^2 \optmlOperScalar^2\|\zz\|^2 = 2 \hat{\sigma}_x^2 \optmlOperScalar^2 \sum_{j = 1}^{m} \frac{1}{w_j^2} .
\end{equation}
To determine the flattest minima, we need to seek for the weights that minimize $\lambda_{\max}\big( \Hess_\weightsPointVec \big)$. This boils down to solving the constrained optimization problem
\begin{equation}
\min_{\weightsPointVec \in \R^m } \   \sum_{j = 1}^{m} \frac{1}{w_j^2} \qquad \text{s.t.} \qquad \prod_{j=1}^{m} w_j = \optmlOperScalar.
\end{equation}

As we show in App.~\ref{appendix:ScalarCaseWidestMinimum}, the minimum of this problem is attained when \( | w_1 | = | w_2 | = \cdots = | w_m | \), so that the set of flattest minima is given by
\begin{equation}\label{eq:WidestMinimaScalar}
\wideminset = \Big\{ \weightsPointVec \ :  \ |w_j| = | \optmlOperScalar |^{\frac{1}{m}}, \,\prod_{j = 1}^{m} \text{sgn}(w_j) =  \text{sgn}( \optmlOperScalar ) \Big\}.
\end{equation}
Substituting $|w_j| = | \optmlOperScalar |^{1/m}$ into \eqref{eq:LambdaMaxScalar}, we obtain that the sharpness of the flattest minima is given by
\begin{equation}\label{eq:LargestWidthScalar}
\min_{\weightsPointVec\in\RR^m} \lambda_{\max}( \Hess_\weightsPointVec) = 2  m \hat{\sigma}_x^2 \optmlOperScalar^{2(1-\frac{1}{m})}.
\end{equation}
Note that although there exist infinitely many global minima, there are far fewer flattest minima. Specifically, we see that for scalar networks, \( \wideminset \) is a discrete set of cardinality \( 2^{m-1} \). Geometrically speaking, within each connected valley of global minima, we have only one flattest minimum point. This property carries over to the vector case, in the sense that $ \wideminset$ is always a set of measure zero within $\minset$.

This simple exercise of analyzing scalar networks already reveals several interesting properties of flat minima.
\begin{enumerate}
	\item \textbf{Balancedness.} Note from \eqref{eq:WidestMinimaScalar} that the flattest minima correspond to networks, which are balanced in the sense that all their layers have the same weight magnitude. This property turns out to break in higher dimensions. However, as we will see, the flattest solutions are always at least \emph{nearly} balanced, and they exhibit interesting coupling properties.
	
	\item \textbf{Step-size and depth.} Observe from~\eqref{eq:LargestWidthScalar} that the sharpness of the flattest minima scales roughly linearly with the network's depth, $m$. Thus, the deeper the network, the smaller the maximal step-size that allows convergence. As we will see, this property persists in higher dimensions. Interestingly, although this behavior is known \cite{nar2018step}, it has not been previously derived from minima sharpness considerations.
	
	\item \textbf{Valley dimensions.} We saw that the Hessian at a global minimum is always rank-1. This implies that at every minimum point, $m-1$ orthogonal directions point into the valley, whereas only one direction points to an ascent slope. We will see that a similar phenomenon occurs also in higher dimensions.
\end{enumerate}

\begin{figure}[t!]
	\begin{subfigure}[t!]{1.4in}
		\includegraphics[width=1.5in,trim={0 0 0 0},clip]{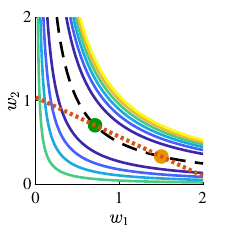}
		\caption{\centering\label{fig:DeceivingExample1}Minima interpolation}
	\end{subfigure}\begin{subfigure}[t!]{2in}
		\includegraphics[width=2in,trim={0 0 0 0},clip]{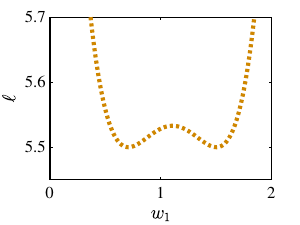}
		\caption{\centering\label{fig:DeceivingExample2}Loss along the dashed line}
	\end{subfigure}
	\caption{A two-layer scalar network example for the misleading nature of minima interpolation.  \protect\subref{fig:DeceivingExample1} We compute the loss along the (dashed) line connecting two global minima, one flattest (green) and one sharp (orange). \protect\subref{fig:DeceivingExample2} Despite having different sharpness in $\R^2$, their sharpness along this 1D cross section are the same. In this setting, this occurs for \emph{any} choice of the non-flattest solution (orange point).}
	\label{fig:DeceivingExample}
\end{figure}

Besides providing a glimpse into the nature of flat minima, the analysis of scalar networks also allows to assess the effectiveness of visualization methods. Particularly, it is common practice to visually compare the sharpness of two minima, $\weightsPointVec^{(1)}$ and $\weightsPointVec^{(2)}$, by plotting the loss along the line connecting them \cite{keskar2016large,jastrzkebski2017three}. One expects that a flat minimum would appear flatter also along this 1D cross-section. However, our scalar network analysis reveals that this is typically incorrect. Let us first take a two-layer example. Figure~\ref{fig:DeceivingExample} shows the loss along the line connecting a flattest minimum point $\weightsPointVec^{(1)}$ and a sharper one, $\weightsPointVec^{(2)}$. As can be seen, along this cross-section, both minima have the same sharpness. This is not a result of some particular choice of $\weightsPointVec^{(2)}$. It turns out that for two-layer scalar networks, the minimas' sharpness along this cross section are always the same, regardless of how sharp $\weightsPointVec^{(2)}$ is in practice. For deeper scalar networks, this is not always the case. However, this visualization is still frequently deceiving (see App.~\ref{appendix:ScalarCaseMinimaInterpolation}).
\begin{lemma}\label{lemma:MinimaInterpolation}
	Consider a scalar linear network. Let $\weightsPointVec^{(1)}$ be a flattest minimum and $\weightsPointVec^{(2)}$ be some other minimum that has the same sign pattern as $\weightsPointVec^{(1)}$. If the interpolation visualization shows that $\weightsPointVec^{(2)}$ is sharper than $\weightsPointVec^{(1)}$, then there exists another minimum, $\weightsPointVec^{(3)}$, which the visualization would show is rather flatter than $\weightsPointVec^{(1)}$. 
\end{lemma}
As we empirically show in Sec.~\ref{sec:Experiments}, this phenomenon is common also in non-scalar networks with ReLU activations.

%% file: MainResults.tex
\section{Main Results}\label{sec:MainResults}

\begin{figure*}[t!]
	\begin{subfigure}[t]{0.5\linewidth}
		\includegraphics[width=3.7in,trim={0.3in 0 0 0},clip]{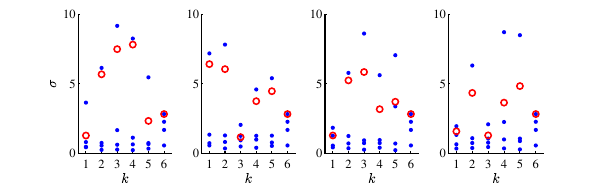}
		\caption{\centering\label{fig:intermediate_gain_linear_sharp}
			Arbitrary global minima}
	\end{subfigure}\begin{subfigure}[t]{0.5\linewidth}
		\includegraphics[width=3.7in,trim={0.3in 0 0 0},clip]{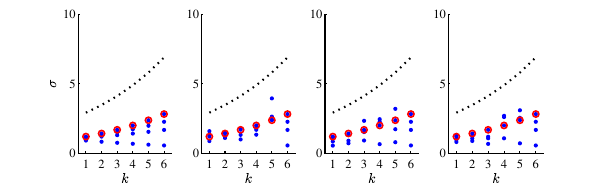}
		\caption{\centering\label{fig:intermediate_gain_linear_wide}
			Flattest global minima}
	\end{subfigure}
	\caption{Intermediate gains versus layer number in deep linear networks. Here we visualize eight randomly chosen implementations of the same end-to-end function $\optmlOper$, where the dimension \( d \) is \( 4 \), and the number of layers \( m \) is always \( 6 \). For each depth $k$, the blue dots depict the singular values of the product of weight matrices from $1$ to \( k \), and the red circle corresponds to the gain of the top singular vector of $\optmlOper$. The black doted line corresponds the bound of Theorem \ref{theorem:PartialProducts}(ii). \protect\subref{fig:intermediate_gain_linear_sharp}~For arbitrary global minima, the intermediate gains can be high. \protect\subref{fig:intermediate_gain_linear_wide}~For flattest solutions, the maximal intermediate gain is well behaved, and $\vv$ is a singular vector of all partial matrix products.}
	\label{fig:intermediate_gain_linear}
\end{figure*}

We now move on to the general case of non-scalar deep linear networks. To simplify notations, we denote
\begin{equation}
\prod_{j=q}^{k} \weights_j \triangleq \weights_k \weights_{k-1} \cdots \weights_q,
\end{equation}
where a product over an empty set (\( q>k \)) is defined to be the identity matrix \( \Identity \). We make the following assumptions.
\begin{enumerate}
	\item[\textbf{A1}] The network has the capacity to implement any linear function from $\R^{d_x}$ to $\R^{d_y}$, namely \( \min\{d_i\} \geq \min\{ d_x, d_y \} \).
	
	\item[\textbf{A2}]  The data is white, namely \( \EmpCovMat{x} = \Identity \).
\end{enumerate}

We begin by identifying the structure of the Hessian matrix at a global minimum (see App.~\ref{appendix:HessianDerivation}).
\begin{lemma}[Hessian structure] \label{lemma:HessianStructure}
	Assume A1. If \( \weightsPointVec \in \minset \), then
	\begin{equation} \label{eq:HessianStructure}
		\Hess_{\weightsPointVec} = 2 \bPhi \bPhi^T,
	\end{equation}
	where \( \bPhi = [\bPhi_1^T , \bPhi_2^T, \ldots, \bPhi_m^T ]^T \), with
	\begin{equation}\label{eq:PhiDef}
		\bPhi_k = \Bigg(\prod_{j=1}^{k-1} \weights_j\EmpCovMat{x}^{\frac{1}{2}}  \Bigg) \otimes \Bigg( \prod_{i=k+1}^{m} \weights_i \Bigg)^T.
	\end{equation}
	Here \( \otimes \) denotes the Kronecker product.
\end{lemma}
Note that $\bPhi_k$ is a $d_{k}d_{k-1}\times d_x d_y$ matrix. Therefore, $\bPhi$ has only $d_x d_y$ columns, while its number of rows is the total number of parameters in the net, $N=\sum_{k=1}^m d_kd_{k-1}$. This shows that for networks with more than one layer, the Hessian at a global minimum is always rank-deficient. For example, if $d_x=d_y\triangleq d$, then we have from Assumption A1 that $N\geq m d^2$, so that at any global minimum point, only \(d^2 \) orthogonal directions point to a slope (the Hessian's rank), while the rest point into the valley of minima. In other words, the dimension of the valley is at least $(1-1/m)$ of the ambient dimension~$N$.

In analogy with the scalar setting, we would now like to exploit Lemma~\ref{lemma:HessianStructure} for analyzing the set of flattest minima,~\( \wideminset \). Unfortunately, here it is intractable to derive a closed form expression for \( \lambda_{\max}(\Hess_{\weightsPointVec}) \) at an arbitrary minimum point. Yet, our key observation is that it is still possible to determine the minimal value of \( \lambda_{\max}(\Hess_{\weightsPointVec}) \) over the set of global minima $\minset$, as well as its associated eigenvector. That is, we can deduce the sharpness of the flattest minima, without having an explicit expression for the sharpness of arbitrary minima. We elaborate on the proof technique in Sec.~\ref{sec:Derivation}. Specifically, let $\sigma_{\max}(\optmlOper)$  denote the top singular value of \( \optmlOper \), and let \(  \uu \) and \( \vv \) be its corresponding left and right singular vectors. Then we have the following.

\begin{theorem}[Sharpness of flattest minima] \label{theorem:LamdaMaxMinimalValue}
	Assume A1 and A2. If \( \weightsPointVec \in \wideminset \) then
	\begin{equation} \label{eq:LamdaMaxMinimalValue}
		\lambda_{\max}(\Hess_{\weightsPointVec}) = 2m \times \left(\sigma_{\max} \left( \optmlOper \right)\right)^{2(1-\frac{1}{m})},
	\end{equation}
	and the corresponding eigenvector is \( \bb = \bPhi \big( \vv \otimes \uu \big) \).
\end{theorem}
This result asserts that the flattest minima become sharper as the number of layers increases (their sharpness grows approximately linearly with \( m \) for \(m \gg 1\)). Since a minimum point \( \weightsPointVec^* \) is stable for GD if the step-size satisfies \( \eta \leq 2/ \lambda_{\max} ( \Hess_{\weightsPointVec^*} ) \), we conclude that the maximal step-size allowing convergence satisfies $ \eta_{\max} \leq \frac{1}{m}(\sigma_{\max}( \optmlOper ))^{-2(1-{1}/{m})}$. In other words, the step-size should be taken to be smaller when training deeper models. As mentioned above, this result was also deduced by \citet{nar2018step}, albeit from different considerations (without explicitly analyzing minima sharpness).

Next, we turn to analyze the flattest minima in terms of the gain that signals experience as they propagate through these networks. For general minimum points, the largest end-to-end gain is \( \sigma_{\max}( \optmlOper ) \) (corresponding to the input~$\vv$), but the intermediate gain up to layer $k<m$ is unconstrained, as we can always multiply one weight matrix by $\alpha$ and another by $1/\alpha$ without affecting the end-to-end mapping. Flattest minima, however, have special structures. Two questions are thus in place regarding those solutions: (i)~What gain does~$\vv$ experience up to layer~$k$? (ii)~What is the largest gain that \emph{any signal} can experience up to layer~$k$?
\begin{theorem}[Intermediate gains] \label{theorem:PartialProducts}
Assume A1 and A2. If \( \weightsPointVec \in \wideminset \) then for all $ k $:
	\begin{enumerate}[label=\roman*.]		
		\item \( \vv \) is a right singular vector of  $\  \prod_{j=1}^{k} \weights_j $ with corresponding singular value $ (\sigma_{\max}( \optmlOper))^{\frac{k}{m}} $.
		\item \( \sigma_{\max}( \prod_{j=1}^{k} \weights_j )  \leq \sqrt{m} \times (\sigma_{\max}( \optmlOper )) ^{\frac{k}{m}}  \).
	\end{enumerate}
	Similarly,
	\begin{enumerate}[resume, label=\roman*.]
		\item \( \uu \) is a left singular vector of $\ \prod_{j=k+1}^{m} \weights_j $ with corresponding singular value $ (\sigma_{\max}(\optmlOper))^{1-\frac{k}{m}} $.
		\item \( \sigma_{\max}( \prod_{j=k+1}^{m} \weights_j )  \leq \sqrt{m} \times (\sigma_{\max}( \optmlOper ))^{1-\frac{k}{m}}\).
	\end{enumerate}
\end{theorem}

Figure~\ref{fig:intermediate_gain_linear} illustrates the theorem for six-layer linear networks designed to solve a linear regression problem involving synthetic data (see App.~\ref{appendix:Experiments} for details). Here  $ \weights_j \in \R^{4 \times 4} $ for all layers. The figure depicts eight randomly drawn global minima, four arbitrary and four flattest. The intermediate gain of $\vv$, to which Theorem~\ref{theorem:PartialProducts}(i) refers, is marked by red circles. The bound of Theorem~\ref{theorem:PartialProducts}(ii) is shown as a dotted black line, and the singular values of the partial matrix products are marked by blue dots. As can be seen in Fig.~\ref{fig:intermediate_gain_linear_sharp}, the intermediate gains in arbitrary global minimum solutions can be high. However, in the flattest solutions (Fig.~\ref{fig:intermediate_gain_linear_wide}), the gain that $\vv$ experiences varies gracefully along the net (as $ (\sigma_{\max}( \optmlOper))^{k/m} $), and the maximal gain of any other signal (highest blue point) is never much larger. Finally, we see that~$\vv$ is indeed one of the singular vectors of the partial product matrix up to any depth (as the red circle coincides with one of the blue points in each layer).

In addition to the intermediate gains, it is of interest to analyze  the individual weight matrices. It turns out that in the flattest solutions, the layers exhibit a sort of coupling associated with the signal $\vv$. Specifically, we have the following.
\begin{theorem}[Layer coupling]\label{theorem:SingularVectorsAndValues}
	Assume A1 and A2. Denote \( \rr_k = \prod_{j=1}^{k-1} \weights_j \vv \), \(\,\, \qq_k = (  \prod_{j=k+1}^{m} \weights_j)^T \uu \), and write $\bar{\rr}_k  = \rr_k/ \| \rr_k \|$, $\,\,\bar{\qq}_k = \qq_k / \| \qq_k \| $.
	If \( \weightsPointVec \in \wideminset \) then for all~$ k $:
	\begin{enumerate}[label=\roman*.]
		\item \( \bar{\qq}_k \)  and \( \bar{\rr}_k  \) are a pair of left and right singular vectors of \( \weights_k \) with corresponding singular value \( (\sigma_{\max} ( \optmlOper))^{\frac{1}{m}} \).
		\item \label{theorem:SingularVectorsAndValues2} These vectors are coupled in the sense that \( \bar{\rr}_{k+1} =  \bar{\qq}_k  \).
	\end{enumerate}
\end{theorem}
{\bf{Remark:}} From Theorem~\ref{theorem:PartialProducts}, \( \| \rr_k \| = (\sigma_{\max} ( \optmlOper ))^{(k-1)/m}  \) and \( \| \qq_k \| = (\sigma_{\max}( \optmlOper))^{1-k/m} \).

Theorem~\ref{theorem:SingularVectorsAndValues} indicates that in flattest minimum networks, there forms a distinct path from input to output that is exclusively dedicated to the signal $\vv$. Specifically, when such a network operates on \( \vv \), the input to each layer is a singular vector of that layer, with singular value \(  (\sigma_{\max}( \optmlOper))^{1/m}  \). Note that this singular value is not necessarily the maximal one of each layer, but it must exist in all matrices. Now, since the input of each layer is a singular vector, so is its output. Therefore, we have that consecutive layers in the network have a singular vector in common, where a left singular vector of one matrix matches a right singular vector of the next.

We saw that if $\weightsPointVec\in\wideminset$, then one of the singular values of each weight matrix must equal \(  (\sigma_{\max}( \optmlOper))^{1/m}  \). One may wonder whether the other direction is also true. As we now show, if the singular value \(  (\sigma_{\max}( \optmlOper))^{1/m}  \) not only exists, but is also the \emph{largest} one of each matrix, then the network is necessarily a flattest minimum.
\begin{theorem}[Sufficient condition] \label{theorem:SufficientCondition}
	Assume A1 and A2. If a solution \( \weightsPointVec \in \minset \) satisfies \( \sigma_{\max}(\weights_k) = (\sigma_{\max}(\optmlOper))^{\frac{1}{m}} \) for all \( k \), then necessarily \( \weightsPointVec \in \wideminset \).
\end{theorem}
Observe that these cases are not rare, in the sense that they form a set of nonzero measure within~$\Omega_0$.

%% file: Derivations.tex
\section{Proof Outline for Theorem~\ref{theorem:LamdaMaxMinimalValue}} \label{sec:Derivation}

Our results in theorems~\ref{theorem:LamdaMaxMinimalValue}-\ref{theorem:SufficientCondition} hinge on the ability to characterize the flattest minima without having an explicit expression for the top eigenvalue of the Hessian at an arbitrary minimum point. In this section we present an outline of the proof of Theorem~\ref{theorem:LamdaMaxMinimalValue}, which illustrates how we go about this, and lays the basis for the proofs of the other theorems.

Note from \eqref{eq:fDef} that \( \weightsPointVec \) is a concatenation of the vectorizations of the weights matrices \( \{ \weights_j \} \). That is, denoting \( \weightsPointVec_j = \vectorize{\weights_j} \), we have that \( \weightsPointVec = [\weightsPointVec_1^T, \weightsPointVec_2^T, \ldots, \weightsPointVec_m^T]^T \). Therefore, the Hessian has the following block structure,
\begin{equation}\label{eq:HessDefinition}
\! \! \! \Hess_{\! \weightsPointVec} \! = \! \!	\begin{bmatrix}
		 \scndPrtDev{\weightsPointVec}{1}{1}  & \scndPrtDev{\weightsPointVec}{1}{2}  & \dots & \scndPrtDev{\weightsPointVec}{1}{m} \\
		\scndPrtDev{\weightsPointVec}{2}{1} & \scndPrtDev{\weightsPointVec}{2}{2} & \dots  & \scndPrtDev{\weightsPointVec}{2}{m} \\
		\vdots            					& \vdots            				  & \ddots & \vdots 							 \\
		\scndPrtDev{\weightsPointVec}{m}{1} & \scndPrtDev{\weightsPointVec}{m}{2} & \dots  & \scndPrtDev{\weightsPointVec}{m}{m}
	\end{bmatrix} \! \!
	\loss(\weightsPointVec), \! \!
\end{equation}
where we use denominator-layout notation. In App.~\ref{appendix:HessianDerivation} we show that if $ \weightsPointVec \in \minset $, then the \( (i,j )\)th block is given by
\begin{equation}
	\frac{\partial^2}{\partial \weightsPointVec_i \partial \weightsPointVec_j} \loss(\weightsPointVec) = 2 \bPhi_i \bPhi_j^T,
\end{equation}
where \( \bPhi_i \) is defined in \eqref{eq:PhiDef}. This implies that we can write \(  \Hess_{\weightsPointVec} = 2\bPhi \bPhi^T \), where \( \bPhi = [\bPhi_1^T , \bPhi_2^T, \ldots, \bPhi_m^T ]^T \).

To study the maximal eigenvalue of the Hessian, we will be rather looking at the matrix \( \HessTag_{\weightsPointVec} = 2 \bPhi^T \bPhi \), whose nonzero eigenvalues coincide with those of \( \Hess_{\weightsPointVec} \). Particularly,
\begin{equation} \label{eq:eigenvalueVectorForm}
	\lambda_{\max}\big( \Hess_{\weightsPointVec} \big) = \lambda_{\max}\big(\HessTag_{\weightsPointVec}\big) = \max_{ \| \bb \| = 1 } 2 \| \bPhi \bb \|^2 .
\end{equation}
Using the fact that $\| \bPhi \bb \|^2 = \sum_{k=1}^m \|\bPhi_k \bb\|^2$, together with properties of the Kronecker product (that appears in the definition of $\bPhi_k$), the right side of \eqref{eq:eigenvalueVectorForm} can be written as 
\begin{equation} \label{eq:eigenvalueMatrixForm}
\max_{ \| \BB \|_{\rm F} = 1 } 2 \sum_{k = 1}^{m} \Big\| {\Big( \prod_{i=k+1}^{m}\!\!\! \weights_i \Big)^{\! \!T} \! \BB  \EmpCovMat{x}^{\frac{1}{2}} \Big(\prod_{j=1}^{k-1} \weights_j  \Big)^T }  \Big\|_{\rm F}^2 ,
\end{equation}
where \( \bb = \vectorize{\BB} \) (see App.~\ref{appendix:MaxEigenvalCanonicalSol}). Obtaining a closed form solution to this optimization problem seems intractable. However, recall that our goal is merely to find the minimal value of $\lambda_{\max}( \Hess_{\weightsPointVec} )$ over \( \weightsPointVec \in \minset \). This corresponds to a minimax optimization problem over $ \weightsPointVec$ and $\BB$, where the minimum is taken over \( \weightsPointVec \in \minset \) and the maximum over $\BB\in\{\BB\in\R^{d_y\times d_x}: \| \BB \|_{\rm F} = 1\}$.

Our solution approach consists of two steps. First, we bound the objective from below using an expression that is independent of \( \weightsPointVec \). Then, we show that there exists a particular choice of \( \weightsPointVec \in \minset \) that achieves the lower bound. This proves that our bound is in fact the minimax value (\ie the minimal value of $\lambda_{\max}( \Hess_{\weightsPointVec} )$ over $\minset$). To this end, we make use of the following lemma (see proof in App.~\ref{appendix:LowerBoundLemmaProof}).
\begin{lemma} \label{lemma:myInequality}
	Let $ \{ \bPsi_k \}_{k=1}^m  $ be a set of matrices such that \( \bPsi_k \in \R^{d_{k} \times d_{k-1}} \), then
	\begin{equation}
		\sum_{k = 1}^{m} \| \bPsi_k \|_{\rm F}^2 \geq  m \,\Bigg( \Big\| \prod_{k = 1}^{m}\bPsi_k \Big\|_{2} \Bigg)^{\frac{2}{m}},
	\end{equation}
	where \( \| \cdot \|_{2} \) is the matrix norm induced by the \( \ell_2 \) vector norm (\ie the maximal singular value of the argument).
\end{lemma}
This lemma implies that the objective in \eqref{eq:eigenvalueMatrixForm} can be lower-bounded by
\begin{equation}\label{eq:RawLemmaBoundResult}
	2 m \Bigg(  \Big\| \prod_{k = 1}^{m}  {\Big( \prod_{i=k+1}^{m} \!\!\!\weights_i \Big)^T \BB \EmpCovMat{x}^{\frac{1}{2}}\Big(\prod_{j=1}^{k-1} \weights_j  \Big)^T } \Big\|_2 \Bigg)^{\frac{2}{m}}.
\end{equation}
Let us write out explicitly two consecutive terms of the outer product,
\begin{equation}\label{eq:MultiplicationUnrolling}
\underbrace{\! \Big(\!\!\! \prod_{i=q+2}^{m} \!\!\!\weights_i \Big)^{\!T}\!\BB\EmpCovMat{x}^{\frac{1}{2}}\Big(\prod_{j=1}^{q} \!\weights_j  \Big)^{\!T} \!  }_{ k = q+1 } \
\underbrace{ \! \Big(\!\!\! \prod_{i=q+1}^{m} \!\!\!\weights_i \Big)^{\!T}\! \BB \EmpCovMat{x}^{\frac{1}{2}}\Big(\prod_{j=1}^{q-1}\! \weights_j  \Big)^{\!T} \! }_{ k = q }
\end{equation}
It is easy to see that the product of the two terms in the middle equals \( (\prod_{j=1}^{m} \weights_j )^T  \), which in turn equals $\optmlOper^T$ for global minima. Therefore, if we unwrap the entire outer product, we get $\optmlOper^T$ in between every two appearances of $\BB\EmpCovMat{x}^{\frac{1}{2}}$, so that \eqref{eq:RawLemmaBoundResult} reduces to
\begin{equation}
\nu(\BB) \triangleq 2m \,  \Big\| \Big(  \BB \EmpCovMat{x}^{\frac{1}{2}}  \optmlOper^T \Big)^{m-1} \BB \EmpCovMat{x}^{\frac{1}{2}} \Big\|_2^{\frac{2}{m}}.
\end{equation}
To recap, we have that if \(  \weightsPointVec \) is a global minimum, then \( \lambda_{\max}( \Hess_{\weightsPointVec} ) \geq \max \nu(\BB) \) s.t.~\( \|  \BB  \|_{\rm F} =1 \). For the special case $\EmpCovMat{x} = \Identity$ (Assumption A2), we show in App.~\ref{appendix:MaxValueLowerBound} that
\begin{equation}\label{eq:LowerBoundVal}
	\max_{ \|  \BB  \|_{\rm F} =1 } \nu(\BB) = 2m \left(\sigma_{\max} ( \optmlOper ) \right)^{2(1-\frac{1}{m})}.
\end{equation}
We have thus obtained a lower-bound on $\lambda_{\max}( \Hess_{\weightsPointVec})$, which is independent of $\weightsPointVec$.

\begin{figure}[t!]
	\begin{subfigure}[t]{\columnwidth}
		\centering
		\includegraphics[width=3.5in,trim={0cm 0 0cm 0},clip]{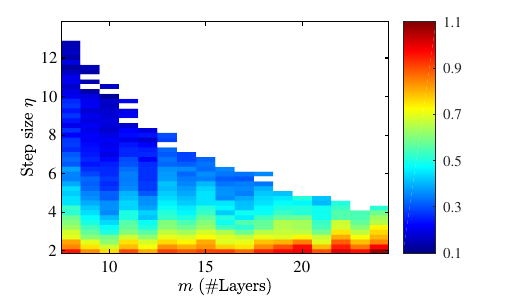}
		\caption{\label{fig:Theorem1_a}\centering{Minima sharpness vs.~step size and network depth}}
	\end{subfigure}
	\begin{subfigure}[t]{\columnwidth}
		\centering
		\includegraphics[width=3in,trim={0cm 0 0cm 0},clip]{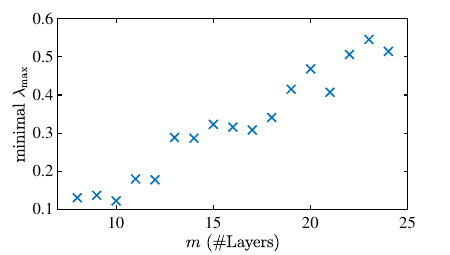}
		\caption{\label{fig:Theorem1_b}\centering{Sharpness of the flattest minima vs.~network depth}}
	\end{subfigure}
	\caption{Sharpness of minima obtained with identity initialization in fully connected ReLU networks trained to denoise MNIST digits. \protect\subref{fig:Theorem1_a}~The color of each tile corresponds to the sharpness of the minimum to which SGD converged for a particular step size and network depth~$m$. White tiles correspond to non-converged trainings. We can see that larger step sizes lead to flatter minima, and that the maximal step size allowing convergence behaves as $ 1/m $. \protect\subref{fig:Theorem1_b}~Here, we see that the sharpness of the flattest minimum (bluest tile) increases roughly linearly with $ m $, as Theorem~\protect\ref{theorem:LamdaMaxMinimalValue} predicts.}
	\label{fig:Theorem1}
\end{figure}

\begin{figure}[t!]
	\begin{subfigure}[t]{\columnwidth}
		\centering
		\includegraphics[width=3.5in,trim={0cm 0 0cm 0},clip]{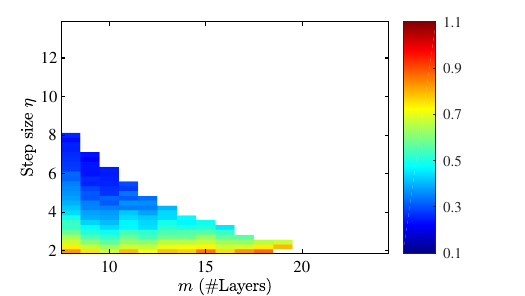}
		\caption{\label{fig:Lemma1_a}\centering{Minima sharpness vs.~step size and depth (random init.)}}
	\end{subfigure}
	\begin{subfigure}[t]{\columnwidth}
		\centering
		\includegraphics[width=3in,trim={0cm 0 0cm 0},clip]{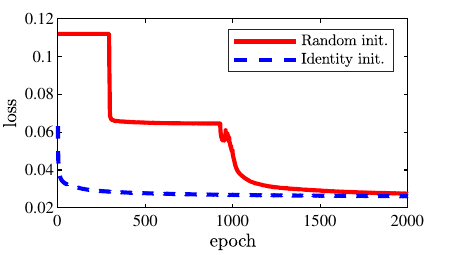}
		\caption{\label{fig:Lemma1_b}\centering{Training plot}}
	\end{subfigure}
	\caption{Sharpness of minima obtained with random initialization. \protect\subref{fig:Lemma1_a}~As opposed to identity initilization (Fig.~\ref{fig:Theorem1}), here the maximal step sizes allowing convergence are smaller, and the minima to which SGD converges are sharper. This aligns with the prediction of Lemma~\ref{lemma:WidestMinimaConvergence}. \protect\subref{fig:Lemma1_b}~We plot the progression of the training loss with random and identity initializations, for an $ 18 $ layer network with step size $\eta = 2.25$. The graph demonstrates that SGD converges faster when initialized at identity.}
	\label{fig:Lemma1}
\end{figure}

We now determine a particular solution achieving the bound. Denote the SVD of \( \optmlOper \) by \( \UU  \SSS \VV^T \) and let \(\weightsPointVec^* \in \minset \) be
\begin{equation}\label{eq:CanonicalSolution}
	\weights^*_m = \UU \SSS_m^{\frac{1}{m}}, \quad  \weights^*_j = \SSS_j^{\frac{1}{m}}, \quad   \weights^*_1 = \SSS_1^{\frac{1}{m}} \VV^T.
\end{equation}
Here we slightly abuse the notation \(\SSS_j^{1/m}\) to denote a $d_{j} \times d_{j-1}$ diagonal matrix whose $k$th diagonal entry is $(\sigma_k(\optmlOper))^{1/m}$, the $k$th largest singular value of~$\optmlOper$. Note that for this particular solution, all the weight matrices have the same set of nonzero singular values, which are precisely the $m$th roots of the singular values of $\optmlOper$. As we show in App.~\ref{appendix:MaxEigenvalCanonicalSol}, for this  solution it is rather easy to compute the Hessian's top eigenvalue, which turns out to equal
\begin{equation}
	\lambda_{\max} ( \Hess_{\weightsPointVec^*} ) = 2m\, (\sigma_{\max} ( \optmlOper ))^{2(1-\frac{1}{m})}.
\end{equation}
Since $\lambda_{\max} ( \Hess_{\weightsPointVec^*} )$ achieves the lower-bound \eqref{eq:LowerBoundVal}, this bound must be the minimal value of $\lambda_{\max}( \Hess_{\weightsPointVec}) $. We have thus established that
\begin{equation}
\min_{\weightsPointVec\in\minset}\lambda_{\max}( \Hess_{\weightsPointVec} ) = 2m\,(\sigma_{\max} ( \optmlOper ))^{2(1-\frac{1}{m})},
\end{equation}
which completes the proof for $ \lambda_{\max}(\Hess_{\weightsPointVec})$. The proof for the corresponding eigenvector can be found in App.~\ref{appendix:TopEigenvectorProof}.

Two comments are in place. First, note that as a byproduct, we obtained that the solution~\eqref{eq:CanonicalSolution} is a flattest global minimum. This is actually a special case of Theorem~\ref{theorem:SufficientCondition}, which applies also to non-diagonal weight matrices, and to matrices whose singular values do not all coincide with the $m$th roots of the singular values of $\optmlOper$. Namely, according to Theorem~\ref{theorem:SufficientCondition}, merely requiring that \( \sigma_{\max}(\weights_k) = (\sigma_{\max}(\optmlOper))^{1/m} \) for all~\( k \), already guarantees that a minimum is flattest. Second, although we focused on the case \( \EmpCovMat{x} = \Identity \), we conjecture that $\min_{ \weightsPointVec \in \minset } \lambda_{\max}( \Hess_{\weightsPointVec} ) = \max_{ \|  \BB  \|_{\rm F} =1 } \nu(\BB)$ also for arbitrary~$\EmpCovMat{x}$. However, in the general setting, there is no closed form solution for the maximization over $\BB$, so that its study seems to allow no further insight.

\begin{figure*}[t!]
	\begin{subfigure}[t]{0.5\linewidth}
		\includegraphics[width=3.7in,trim={0.3in 0 0 0},clip]{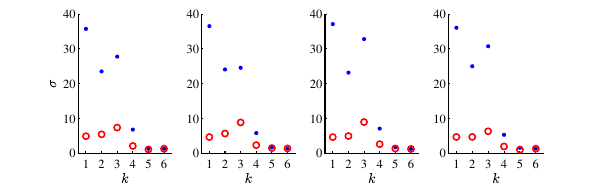}
		\caption{\label{fig:intermediate_gain_nonlinear_sharp}Sharp minima (Adam)}
	\end{subfigure}\begin{subfigure}[t]{0.5\linewidth}
		\includegraphics[width=3.7in,trim={0.3in 0 0 0},clip]{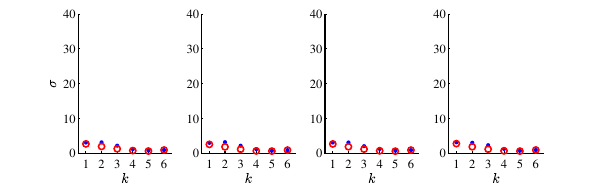}
		\caption{\label{fig:intermediate_gain_nonlinear_wide}Flat minima (SGD with large step-size)}
	\end{subfigure}
	\caption{Intermediate gains versus layer number for six-layer fully connected ReLU networks, trained to denoise MNIST digits. The maximal gain from input to each layer is marked by a blue dot (analogous to the highest blue dot in Fig.~\ref{fig:intermediate_gain_linear}). The red circles correspond to the intermediate gain of the signal that experiences the maximal end-to-end amplification. \protect\subref{fig:intermediate_gain_nonlinear_sharp}~The gains in the sharp minimum solutions reached by Adam, are large. \protect\subref{fig:intermediate_gain_nonlinear_wide}~The gains in the flat minimum solutions found by SGD, are significantly more balanced.}
	\label{fig:intermediate_gain_nonlinear}
\end{figure*}

\begin{figure}[t!]
	\begin{subfigure}[t]{\columnwidth}
		\centering
		\includegraphics[width=3.5in,trim={0cm 0 0cm 0},clip]{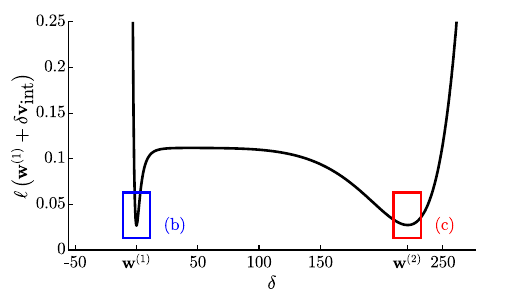}
		\caption{\label{fig:NonlinearDeceivingExampleInterpolation}\centering{The loss along the line connecting $ \weightsPointVec^{(1)} $ and $ \weightsPointVec^{(2)} $}}
	\end{subfigure}\\
	
	\vspace{0.3cm}
	
	\begin{subfigure}[t]{0.51\columnwidth}
		\includegraphics[width=1.75in,trim={0cm 0 0cm 0},clip]{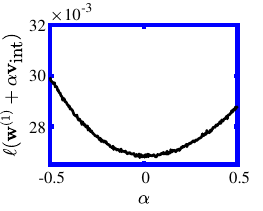}
		\caption{\label{fig:NonlinearDeceivingExampleZoom1}\centering{Zoom-in around $ \weightsPointVec^{(1)} $}}
	\end{subfigure}\begin{subfigure}[t]{0.49\columnwidth}
		\includegraphics[width=1.75in,trim={0cm 0 0cm 0},clip]{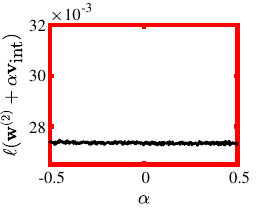}
		\caption{\label{fig:NonlinearDeceivingExampleZoom2}\centering{Zoom-in around $ \weightsPointVec^{(2)} $}}
	\end{subfigure}\\
	
	\vspace{0.3cm}
	
	\begin{subfigure}[t]{0.51\columnwidth}
		\includegraphics[width=1.75in,trim={0cm 0 0 0},clip]{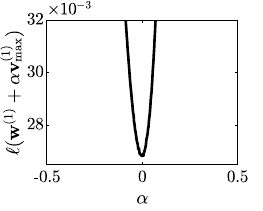}
		\caption{\label{fig:NonlinearWidth1}\centering{Sharpest direction of $ \weightsPointVec^{(1)} $}}
	\end{subfigure}\begin{subfigure}[t]{0.49\columnwidth}
		\includegraphics[width=1.75in,trim={0cm 0 0 0},clip]{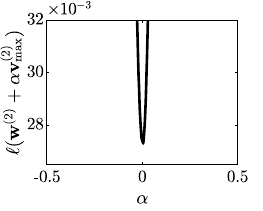}
		\caption{\label{fig:NonlinearWidth2}\centering{Sharpest direction of $ \weightsPointVec^{(2)} $}}
	\end{subfigure}
	\caption{One dimensional cross-sections of the loss landscape. \protect\subref{fig:NonlinearDeceivingExampleInterpolation}~The loss along the line connecting $ \weightsPointVec^{(1)} $ and $ \weightsPointVec^{(2)} $, the solutions obtained by SGD and Adam, respectively. Here, the direction vector is $ \vv_{\text{int}} = (\weightsPointVec^{(2)}-\weightsPointVec^{(1)} )/\|{\weightsPointVec^{(2)}-\weightsPointVec^{(1)}}\| $. Along this cross-section, $ \weightsPointVec^{(1)} $ appears to be sharper than $ \weightsPointVec^{(2)} $.  \protect\subref{fig:NonlinearDeceivingExampleZoom1}, \protect\subref{fig:NonlinearDeceivingExampleZoom2}~Close-ups on $\weightsPointVec^{(1)}$ and $\weightsPointVec^{(2)}$, respectively. \protect\subref{fig:NonlinearWidth1}, \protect\subref{fig:NonlinearWidth2}~Cross-sections corresponding to the sharpest direction of $ \weightsPointVec^{(1)} $ and $ \weightsPointVec^{(2)} $, respectively, which show that $\weightsPointVec^{(1)}$ is in fact flatter. Here, $ \vv_{\max}^{(i)} $ is the top eigenvector of the Hessian matrix at $ \weightsPointVec^{(i)}$.}
	\label{fig:NonlinearDeceivingExample}
\end{figure}

%% file: Experiments.tex
\section{Experiments with Nonlinear Networks}\label{sec:Experiments}
Our theoretical results apply to linear networks. Yet, as we now  empirically illustrate, they also nicely capture the behavior of nonlinear networks. To show this, we trained fully connected networks with ReLU activation functions to denoise images of handwritten digits. We used the MNIST dataset \cite{lecun1998mnist} and simulated zero-mean white Gaussian noise of standard deviation $1.25$, where the pixel range of the clean images was $[0,1]$. The input, output, and all intermediate representations had 784 dimensions, so that the total number of parameters was over $600,000\times m$ for an $m$-layer network. We minimized the quadratic loss using SGD without momentum.

We start by demonstrating Theorem~\ref{theorem:LamdaMaxMinimalValue}, which asserts that all minima become sharper as the depth of the network increases.  Figure~\ref{fig:Theorem1} visualizes the result of training networks of different depths using varying step sizes. For each configuration, we measured the top eigenvalue of the Hessian using the power method. Thus, each tile in Fig.~\ref{fig:Theorem1_a} corresponds to a different trained network, where $\lambda_{\max}( \Hess_{\weightsPointVec} )$ is color-coded and white tiles correspond to non-converged runs. We can see that larger step sizes indeed lead to flatter minima (\ie more bluish tiles). Also, notice that the maximal step size allowing convergence behaves as $ 1/m $, in accordance with Theorem~\ref{theorem:LamdaMaxMinimalValue}. Now, for each network depth, we took the lowest measured sharpness, and plotted it against the number of layers in Fig.~\ref{fig:Theorem1_b}. Here we can see that the flattest minima indeed get sharper as the network gets deeper. Particularly, the behavior is roughly linear, as \eqref{eq:LamdaMaxMinimalValue} predicts.

In the experiment above, we used identity initialization, as 
Lemma~\ref{lemma:WidestMinimaConvergence} suggests this should lead to a flat minimum. To verify that this is indeed the case, we repeated the experiment with the initialization of \citet{he2015delving}. As can be seen in Fig.~\ref{fig:Lemma1_a}, in this case SGD indeed converges to sharper minima, and cannot accommodate large step sizes. To further compare these initializations, we plot the loss function during training in Fig.~\ref{fig:Lemma1_b}. We can see that with identity initialization, the loss rapidly convergence already at an early stage, whereas with random initialization it decreases only at later iterations. This shows that identity initialization indeed leads to flatter minima, as Lemma~\ref{lemma:WidestMinimaConvergence} predicts, and that SGD converges faster to flat solutions.

Next, we demonstrate Theorem~\ref{theorem:PartialProducts}. For the purpose of comparing the properties of sharp and flat solutions, we trained a six layer network for the same denoising problem as above, using two different optimization methods: (i)~SGD with a large step size and moderate batch size, a configuration that is known to converge to flat minima \cite{keskar2016large}; (ii)~Adam \cite{kingma2014adam} with a small step size, which can converge to sharp minima \cite{wu2018sgd}. We ran each method with four different random initializations, and calculated the top eigenvalue of the Hessian at the minimum it converged to. We verified this eigenvalue was indeed significantly smaller (roughly $6\times$) for the minima found by SGD. Now, for each network, we estimated the maximal gain that any signal can experience up to each layer. We did so by optimizing the input so as to maximize the norm of the intermediate signal, where we started from 100 different random initializations and chose the maximum over all runs. These gains are analogous to the top singular values of the partial matrix products in the linear setting. As can be seen in Fig.~\ref{fig:intermediate_gain_nonlinear}, these gains (blue dots) tend to be high in the sharp solutions, and quite restrained in the flat ones. A similar behavior is seen for the intermediate gains of the signal that experiences the largest end-to-end amplification (red circles). These are analogous to the intermediate gains of the vector $\vv$ in the linear setting. These behaviors are in accordance with points (ii) and (i) of Theorem~\ref{theorem:PartialProducts}, respectively.

Finally, we illustrate Lemma~\ref{lemma:MinimaInterpolation}. For each of the 16 pairs of flat and sharp minima, we evaluated the loss along the line connecting them. The result for one pair is shown in Fig.~\ref{fig:NonlinearDeceivingExampleInterpolation} (all 16 pairs showed the same behavior). As can be seen, the flat minimum appears to be sharper than the sharp one along this 1D cross-section. To appreciate how distorted this image is, we also plot in figures \ref{fig:NonlinearWidth1} and \ref{fig:NonlinearWidth2} the loss along the sharpest cross-section of each minimum, which visualizes its true sharpness. This illustration confirms that the interpolation visualization is frequently deceiving also for nonlinear networks in high-dimensional settings.

%% file: RelatedWork.tex
\section{Related Work}\label{sec:RelatedWork}

\paragraph{Notions of sharpness}
Many works studied flat minima in neural networks, especially in relation to generalization. These minima are thought to represent simple models, which are less expected to overfit. However, there is no single definition for minimum sharpness. \citet{hochreiter1997flat} defined it as the size of the connected region around the minimum where the training loss remains low. \citet{chaudhari2019entropy} used local entropy as a measure of sharpness. And \citet{keskar2016large} characterized sharpness using the eigenvalues of the Hessian, and proposed an approximation using the maximal loss in an $\epsilon$-neighborhood of the minimum. These notions of sharpness were devised with the purpose of correlating with generalization, although \citet{neyshabur2017exploring} and \citet{dinh2017sharp} argued they often do not suffice for predicting the generalization of solutions. In contrast to these papers, the definition we studied is associated with the stability of the optimizer at the minimum. Thus, whether correlated with generalization or not for nonlinear nets, it is most relevant for the solutions found in practice by SGD.

\paragraph{Balancedness and alignment}
We showed that GD and SGD tend to converge to solutions that are balanced and aligned. Similar results were studied in different contexts. For example, it has been shown that gradient flow maintains the differences between the squared norms of the layers, both in linear networks \cite{arora2018optimization} and in nonlinear models with homogeneous activation functions \cite{du2018algorithmic}. Notice, however, that as opposed to our analysis, these results break for GD with a large step size, as the authors indicated in their work. Interestingly, while our results apply to linear models trained with a quadratic loss, similar phenomena occur in deep linear networks for binary classification trained with a monotonic loss. Specifically, \citet{ji2019gradient} showed that in those cases gradient flow aligns the layers in the sense that the normalized matrices asymptotically equal their rank-$1$ approximations. Additionally, they showed that adjacent rank-$1$ approximations have a singular vector in common, where a left singular vector of one layer asymptotically matches a right singular vector of the next. Nevertheless, note that networks with vector-valued outputs trained for regression, as we analyzed here, exhibit richer and more complex behaviors than models with scalar outputs trained for binary classification.

\paragraph{Visualization of minima sharpness}
It is fairly common to compare the sharpness of two minima by plotting the loss along the line connecting them \cite{keskar2016large,jastrzkebski2017three}. Yet, we are not the first to discuss the limitations of this common practice. For example, \citet{li2018visualizing} argued that this may depict a misleading picture due to unnormalized weights. Instead, they offered to plot the loss on a randomly chosen 2D cross-section, where the perturbation is normalized with respect to the weights. Here, we gave a concrete example along with a proof that the interpolation visualization is deceiving surprisingly often.

%% file: Conclusion.tex
\section{Conclusion}\label{sec:Conclusion}
Gradient descent methods have a bias towards flat minima. In this work, we proved that for linear networks trained with a quadratic loss, these solutions possess unique properties. For example, in flat minima networks, the signal $\vv$ that experiences the largest gain end-to-end, is amplified as moderately as possible by each layer. Moreover, no other signal can experience a significantly larger gain than $\vv$ up to an intermediate layer. Finally, these solutions exhibit a coupling between the layers, which forms a distinct path for~$\vv$ from input to output. While our theoretical results apply to linear networks, our experiments show that these properties are also characteristic of nonlinear networks trained in practice.

%% file: Appendices.tex
\section{Stability of Minima}\label{appendix:Stability}
It is well know that for a $ \beta $-smooth function that is bounded from below, GD with constant step size $ 0<\eta< 2/\beta $ converges to a stationary point. A twice continuously differentiable function $ f $ is  $ \beta $-smooth \emph{if and only if} $ \lambda_{\max}(\nabla^2 f(\weightsPointVec) ) \leq \beta$ for every point $\weightsPointVec \in \R^{N} $. Hence, convergence to a stationary point is guaranteed if $ \lambda_{\max}(\nabla^2 f(\weightsPointVec) ) \leq 2/\eta  $ for all $  \weightsPointVec \in \R^{N} $. This seemingly stringent global requirement can in fact also be replaced by a local one, as shown by \citet{wu2018sgd}. Specifically, they use the following.
\begin{definition}
	Let $ \weightsPointVec^* $ be a stationary point of $ f $. Consider the linearized dynamical system of GD, namely
	\begin{equation}
		\weightsPointVec_{k+1} = \weightsPointVec_{k} - \eta \nabla^2 f(\weightsPointVec^*)( \weightsPointVec_{k} - \weightsPointVec^*).
	\end{equation}
	Then $ \weightsPointVec^* $ is said to be \emph{linearly stable} if there exists a constant $ C \in \R $, such that $ \norm{\weightsPointVec_k} \leq C \norm{\weightsPointVec_0}$ for all $  k > 0  $.
\end{definition}
In other words, $ \weightsPointVec^* $ is linearly stable if once we have arrived near this critical point, we stay around it. In their paper, \citet{wu2018sgd} show that $ \weightsPointVec^* $ is a linearly stable minimizer if
\begin{equation}\label{eq:OriginalStabilityCondition}
	\lambda_{\max}\left(  \big( \Identity -\eta \nabla^2 f(\weightsPointVec^*) \big)^2 \right) \leq 1.
\end{equation}
Note that for all $ i \in \{1,\ldots,N \} $
\begin{align}
	\lambda_{i}\left( \big( \Identity -\eta \nabla^2 f(\weightsPointVec^*) \big)^2 \right)  & = \lambda ^2 _{i}\left( \Identity -\eta \nabla^2 f(\weightsPointVec^*) \right)\nonumber \\
	& = \left(1 - \eta \lambda_{i}\left( \nabla^2 f(\weightsPointVec^*)\right) \right)^2\nonumber\\
	& =  1 - \eta \lambda_{i}\left(\nabla^2 f(\weightsPointVec^*)\right)\left(2- \eta \lambda_{i}\left(\nabla^2 f(\weightsPointVec^*)  \right)\right),
\end{align}
where $ \lambda_{i} $ is the $ i $th largest eigenvalue. Since $ \eta  $ and $ \{ \lambda_{i} \} $ are all nonnegative, it follows that \eqref{eq:OriginalStabilityCondition} is equivalent to
\begin{equation}
	\lambda_{\max}\left(\nabla^2 f(\weightsPointVec^*)\right) \leq \frac{2}{\eta}.
\end{equation}
This results asserts that flat minima are stable solutions for GD. In their paper, they also provide a similar result for stochastic GD (SGD), which shows that the sharpness of a minimum should increase to ensure stability for SGD as well.

\section{Proof of Lemma~\ref{lemma:WidestMinimaConvergence}}\label{appendix:WidestMinimaConvergenceProof}
\citet{nar2018step} showed that under the Lemma's conditions, the weight matrices converge at a linear rate to $ \weights_i = \optmlOper^{\frac{1}{m}} $ for all $i$. According to our Theorem \ref{theorem:SufficientCondition}, this solution is a flattest minimum, thus demonstrating ($ii$).

\citet{arora2018optimization} showed that gradient flow (GF) satisfies
\begin{equation}\label{eq:MatrixEquality}
\weights_i(t)\weights^T_i(t) = \weights^T_{i+1}(t)\weights_{i+1}(t), \quad \forall t\geq 0
\end{equation}
in our setting. Denoting the SVD of $ \weights_i(t) $ by $ \UU_i(t)\SSS_i(t)\VV^T_i(t) $, we thus have that $\UU_i(t)\SSS^2_i(t)\UU^T_i(t) = \VV_{i+1}(t)\SSS^2_{i+1}(t)\VV^T_{i+1}(t)$, which implies that\footnote{The SVD can be non-unique, however there necessarily exists a decomposition satisfying $ \UU_i(t) = \VV_{i+1}(t) $.}
\begin{equation}\label{eq:proofGF1}
\UU_i(t) = \VV_{i+1}(t), \quad \SSS_{i}(t) = \SSS_{i+1}(t), \quad \forall t\geq 0.
\end{equation}
Assume that GF converges to a global minimimum and let $ \UU\SSS\VV^T $ denote the SVD of $\optmlOper$. Since $ \{ \SSS_{i}(t) \}_{i = 1}^m $ are identical, they converge to the same limit, $ \bar{\SSS} $. Let $ \weights_i = \UU_i\bar{\SSS}\VV_i^T $ denote the limit of $\weights_i(t)$. Then, from \eqref{eq:proofGF1}, we have that $ \VV_{i+1}^T \UU_i  = \Identity $ for all $i$. Consequently, 
\begin{equation}
\weights_m \weights_{m-1} \cdots \weights_1 = \UU_{m} \bar{\SSS}^m \VV^T_{1}.
\end{equation} 
But since the left hand side equals $\optmlOper$ by assumption, the right hand side must coincide with the SVD of $\optmlOper$. This means that $ \bar{\SSS} = \SSS^{\frac{1}{m}} $. Again, by Theorem~\ref{theorem:SufficientCondition}, this is a flattest minimum, thus demonstrating ($i$).

\section{Scalar Networks}\label{appendix:ScalarCase}

\subsection{The Set of Flattest Minima}\label{appendix:ScalarCaseWidestMinimum}

As mentioned in the main text, in the scalar case, the end-to-end function \( f_\weightsPointVec(x) \) implemented by the network is given by
\begin{equation}
	f_\weightsPointVec(x) = \prod_{j=1}^{m} w_j x,
\end{equation}
where $ \weightsPointVec = [w_1,w_2,\ldots,w_m]^T $. In our analysis we consider a quadratic loss function, \emph{i.e.}
\begin{equation}\label{eq:scalarObjectiveFunction}
\loss(\weightsPointVec) = \E{\big(y-f_\weightsPointVec(x)\big)^2}.
\end{equation} 
Our goal is to characterize the set of flattest minima of the loss w.r.t. \( \weightsPointVec \). It is well known that the optimal coefficient for linear estimation is given by
\begin{equation} \label{eq:ScalarFirstOrderOptimalityValue}
\optmlOperScalar = \frac{\hat{\sigma}_{xy}}{\hat{\sigma}_x^2},
\end{equation}
where \( \hat{\sigma}_x^2=\Ei{x^2} \) is the empirical second-order moment of \( x \), and \( \hat{\sigma}_{xy}=\Ei{xy} \) is the empirical cross second-order moment between \(x\) and \( y \). Therefore, at any global minimum of \( \loss(\weightsPointVec) \), we have
\begin{equation}\label{eq:FirstOrderOptimalityScalar}
\prod_{j=1}^{m} w_j = \optmlOperScalar.
\end{equation}
To compute the Hessian matrix of \( \loss(\weightsPointVec) \), we first calculate the partial derivative w.r.t. \(  w_k \),
\begin{align}\label{eq:FirstOrderDevScalarCase}
	\frac{\partial }{\partial w_k} \loss(\weightsPointVec) & = \frac{\partial }{\partial w_k}  \E{\big(y-f_\weightsPointVec(x)\big)^2}  = -2\E{\big(y-f_\weightsPointVec(x)\big)\frac{\partial }{\partial w_k}f_\weightsPointVec(x)}\nonumber\\
	& = -2\E{\Big(y-\prod_{j=1}^{m} w_j x \Big)\prod_{j \neq k } w_j x} = 2\Big( \hat{\sigma}_x^2 \prod_{j=1}^{m} w_j  - \hat{\sigma}_{xy}\Big)\prod_{j \neq k } w_j
\end{align}
We now complete the derivation by differentiating \eqref{eq:FirstOrderDevScalarCase} w.r.t. \(  w_q \),
\begin{align}\label{eq:SecondOrderDevScalarCase}
\frac{\partial^2 }{\partial w_q \partial w_k} \loss(\weightsPointVec) & = \frac{\partial }{\partial w_q} \left[ 2\Big( \hat{\sigma}_x^2 \prod_{j=1}^{m} w_j  - \hat{\sigma}_{xy}\Big)\prod_{j \neq k } w_j \right] \nonumber\\
& = 2 \prod_{j \neq k } w_j \frac{\partial }{\partial w_q} \Big( \hat{\sigma}_x^2 \prod_{j=1}^{m} w_j  - \hat{\sigma}_{xy}\Big) + 2 \Big( \hat{\sigma}_x^2 \prod_{j=1}^{m} w_j  - \hat{\sigma}_{xy}\Big) \frac{\partial }{\partial w_q} \prod_{j \neq k } w_j.
\end{align}
Eq.~\eqref{eq:FirstOrderOptimalityScalar} asserts that \( \hat{\sigma}_x^2 \prod_{j=1}^{m} w_j  - \hat{\sigma}_{xy} = 0 \) at global minima, therefore the second term in \eqref{eq:SecondOrderDevScalarCase} vanishes, and we obtain 
\begin{equation}\label{eq:SecondOrderDevScalarCaseRes}
\frac{\partial^2 }{\partial w_q \partial w_k} L(\weightsPointVec) = 2 \prod_{j \neq k } w_j \frac{\partial }{\partial w_q} \Big( \hat{\sigma}_x^2 \prod_{j=1}^{m} w_j  - \hat{\sigma}_{xy}\Big) = 2 \hat{\sigma}_x^2 \Big(  \prod_{j \neq k } w_j \Big) \Big(  \prod_{j \neq q } w_j \Big).
\end{equation}
Hence, using \eqref{eq:ScalarFirstOrderOptimalityValue}, we can express the elements of the Hessian matrix \( \Hess_\weightsPointVec \) as 
\begin{equation}
\big( \Hess_\weightsPointVec \big)_{k,q} = 2 \hat{\sigma}_x^2  \optmlOperScalar^2 \frac{1}{w_k w_q}.
\end{equation}
Let us define the vector \( \zz = [w_1^{-1},w_2^{-1},\ldots,w_m^{-1}]^T \), then the Hessian matrix can be equivalently written as
\begin{equation}\label{eq:ScalarCaseHessian}
\Hess_\weightsPointVec = 2 \hat{\sigma}_x^2  \optmlOperScalar^2 \zz \zz^T.
\end{equation}
This shows that the Hessian is a rank one matrix, which implies that it has only one nonzero eigenvalue, with a corresponding eigenvector \( \zz \). Therefore,
\begin{equation}
\lambda_{\max}\big( \Hess_\weightsPointVec \big) \zz = \Hess_\weightsPointVec \zz = 2 \hat{\sigma}_x^2  \big(\optmlOperScalar\big)^2 \zz \zz^T \zz = 2 \hat{\sigma}_x^2  \big(\optmlOperScalar\big)^2 \norm{\zz}^2 \zz,
\end{equation}
so that the eigenvalue is given by
\begin{equation}
\lambda_{\max}\big( \Hess_\weightsPointVec \big) = 2 \hat{\sigma}_x^2 \optmlOperScalar^2\|\zz\|^2 =  2 \hat{\sigma}_x^2 \optmlOperScalar^2 \sum_{j = 1}^{m} \frac{1}{w_j^2} .
\end{equation}
To determine the sharpness of the flattest minima, we need to solve the problem
\begin{equation}\label{eq:OptimizationScalar}
\min_{\weightsPointVec \in \R^m } \   \lambda_{\max}\big( \Hess_\weightsPointVec \big) \qquad \text{s.t.} \qquad \prod_{j=1}^{m} w_j = \optmlOperScalar.
\end{equation}
By the inequality of the arithmetic and geometric means, we have that for any feasible point \( \weightsPointVec \) 
\begin{equation} \label{eq:GeoArithmInq}
\sum_{j = 1}^{m} \frac{1}{w_j^2}  \geq  m \times \Bigg( \prod_{j = 1}^{m}  \frac{1}{w_j^2} \Bigg)^{\frac{1}{m}} = m \times  \optmlOperScalar ^{-\frac{2}{m}}.
\end{equation}
Therefore, for all feasible points,
\begin{equation}
	\lambda_{\max}\big( \Hess_\weightsPointVec \big) \geq  2  m \hat{\sigma}_x^2 \optmlOperScalar^{2(1-\frac{1}{m})}  .
\end{equation}
On the other hand, this inequality can be achieved by setting \( | w_1 | = | w_2 | = \cdots = | w_m | \). This shows that the right-hand-side is precisely the sharpness of the flattest minimum, so that 
\begin{equation}
	\wideminset = \Big\{ \weightsPointVec \in \R^{m} \ :  \prod_{j = 1}^{m} \text{sgn}(w_j) =  \text{sgn}\big( \optmlOperScalar \big)  \quad  \text{and} \quad\ |w_j| = \sqrt[m]{|\optmlOperScalar|} \quad \forall j   \Big\}.
\end{equation}

\subsection{Proof of Lemma \ref{lemma:MinimaInterpolation}}\label{appendix:ScalarCaseMinimaInterpolation}
In this section we examine the behavior of the loss function on a line connecting two minima.

\begin{claim} \label{claim:ConnectingLine}
	Assume that $ \optmlOperScalar > 0  $ and let $ \weightsPointVec^{(1)} $ and $ \weightsPointVec^{(2)} $ be minimizers of the objective \eqref{eq:scalarObjectiveFunction} in $ \R^m_{+} $. Then, along the line connecting $ \weightsPointVec^{(1)} $ and $ \weightsPointVec^{(2)} $, the loss function will appear sharper around $ \weightsPointVec^{(1)} $ than around $ \weightsPointVec^{(2)} $ if
	\begin{equation}
		\sum_{i = 1}^{m} \frac{w_i^{(2)}}{w_i^{(1)}} >  \sum_{i = 1}^{m} \frac{w_i^{(1)}}{w_i^{(2)}}.
	\end{equation}
\end{claim}

\begin{proof}
The direction vector of the connecting line is $ \bAlpha = (\weightsPointVec^{(1)} - \weightsPointVec^{(2)}) /\normi{\weightsPointVec^{(1)} - \weightsPointVec^{(2)}} $. Along this direction, the behavior of the loss function around $ \weightsPointVec^{(i)} $ is given by
\begin{equation}
	\loss(\weightsPointVec^{(i)} + \eta \bAlpha ) \approx \loss(\weightsPointVec^{(i)})+ \eta  \bAlpha^T \nabla \loss(\weightsPointVec^{(i)}) + \frac{\eta^2 }{2}\bAlpha^T \Hess(\weightsPointVec^{(i)}) \bAlpha.
\end{equation}
Since $ \nabla \loss(\weightsPointVec^{(1)}) = \nabla \loss(\weightsPointVec^{(2)}) = \zeroVec $ and $ \loss(\weightsPointVec^{(1)}) = \loss(\weightsPointVec^{(2)})  $, the loss function will appear sharper around $ \weightsPointVec^{(1)} $ than around $ \weightsPointVec^{(2)} $, if $ \bAlpha^T \Hess(\weightsPointVec^{(1)}) \bAlpha > \bAlpha^T \Hess(\weightsPointVec^{(2)}) \bAlpha $. From  \eqref{eq:ScalarCaseHessian}, this condition is equivalent to $\|\bAlpha^T\zz^{(1)}\|^2>\|\bAlpha^T\zz^{(2)}\|^2$, or more explicitly,
\begin{equation}
	\left(  \left(\weightsPointVec^{(1)} - \weightsPointVec^{(2)}\right)^T \zz^{(1)} \right)^2 > 
	\left(  \left(\weightsPointVec^{(1)} - \weightsPointVec^{(2)}\right)^T \zz^{(2)} \right)^2.
\end{equation}
Since \( \zz = [w_1^{-1},w_2^{-1},\ldots,w_m^{-1}]^T \), this inequality can be written as 
\begin{equation} \label{eq:AbsCriterion}
\left| \sum_{i = 1}^{m}\frac{w_i^{(2)}}{w_i^{(1)}} - m \right| > 
\left|  \sum_{i = 1}^{m}\frac{w_i^{(1)}}{w_i^{(2)}} - m  \right|.
\end{equation}
Note that
\begin{equation}
	\sum_{i = 1}^{m}\frac{w_i^{(2)}}{w_i^{(1)}}
	\geq m \sqrt[m]{\prod_{i = 1}^{m}\frac{w_i^{(2)}}{w_i^{(1)}} }
	= m \sqrt[m]{\frac{\prod_{i = 1}^{m} w_i^{(2)}}{\prod_{j = 1}^{m} w_j^{(1)}}}
	= m \sqrt[m]{\frac{\optmlOperScalar}{\optmlOperScalar}}
	= m.
\end{equation}
Similarly, $ \sum_{i = 1}^{m}\frac{w_i^{(1)}}{w_i^{(2)}}
\geq m $. Therefore, \eqref{eq:AbsCriterion} can be reduced to
\begin{equation}
\sum_{i = 1}^{m} \frac{w_i^{(2)}}{w_i^{(1)}} >  \sum_{i = 1}^{m} \frac{w_i^{(1)}}{w_i^{(2)}}.
\end{equation}
\end{proof}

Notice that the loss function is symmetric in a sense that if we flip the sign of two scalar layers, then it remains the same. Therefore, without loss of generality, we can restrict our analysis to a single orthant. 
Let $ \optmlOperScalar >0  $, and $ \weightsPointVec^{(1)} $ be the flattest minimum in $ \R^m_{+} $, \ie $ w_i^{(1)} = \optmlOperScalar^{{1}/{m}} $ for all $ i \in \{ 1, \ldots, m \} $. Given a second minimum $ \weightsPointVec^{(2)} \in \R^m_+ $ for which the connecting line between $ \weightsPointVec^{(1)} $ and $ \weightsPointVec^{(2)} $ is loyal to the true sharpness, we can construct a third solution $ \weightsPointVec^{(3)} $ that will appear deceivingly flatter than $ \weightsPointVec^{(1)} $ over their connecting line. Specifically, let us set
\begin{equation}
	w^{(3)}_i = \frac{ \left( w^{(1)}_i \right)^2 }{w^{(2)}_i}.
\end{equation}
Clearly, $ \weightsPointVec^{(3)} $ is a global minimum as $ \prod_{i = 1}^m w^{(3)}_i = \optmlOperScalar  $. Since $ \weightsPointVec^{(2)} $ appears sharper than $ \weightsPointVec^{(1)} $ along their connecting line, then according to Claim~\ref{claim:ConnectingLine} we have
\begin{equation}
\sum_{i = 1}^{m} \frac{w_i^{(1)}}{w_i^{(2)}} >  \sum_{i = 1}^{m} \frac{w_i^{(2)}}{w_i^{(1)}}.
\end{equation}
Thus,
\begin{equation}
	\sum_{i = 1}^{m} \frac{w_i^{(3)}}{w_i^{(1)}}  = \sum_{i = 1}^{m} \frac{w_i^{(1)}}{w_i^{(2)}} > \sum_{i = 1}^{m} \frac{w_i^{(2)}}{w_i^{(1)}} = \sum_{i = 1}^{m} \frac{w_i^{(1)}}{w_i^{(3)}}.
\end{equation}
Therefore, by Claim~\ref{claim:ConnectingLine}, $ \weightsPointVec^{(1)} $ appears sharper than $ \weightsPointVec^{(3)} $ along their connecting line.

In the special case of two layer networks ($ m = 2 $), we have that for any minimizer, $ w_2^{(i)} = \optmlOperScalar/w_1^{(i)} $. Hence,
\begin{equation}
\frac{w_1^{(1)}}{w_1^{(2)}} + \frac{w_2^{(1)}}{w_2^{(2)}} = \frac{w_2^{(2)}}{w_2^{(1)}} + \frac{w_1^{(2)}}{w_1^{(1)}}.
\end{equation}
This means that the minima will appear equally sharp.

\section{Proof of Lemma~\ref{lemma:HessianStructure}}\label{appendix:HessianDerivation}
In this section we derive the Hessian matrix defined in \eqref{eq:HessDefinition} at a global minimum point, \ie for \( \weightsPointVec \in \minset \). Throughout this section we will be using the following properties of the Kronecker product. For any matrices \( \MM_1,\MM_2,\MM_3,\MM_4 \),
\begin{gather*}
	\vectorize{\MM_1 \MM_2 \MM_3}  = \big( \MM_3^T \otimes \MM_1 \big) \vectorize{ \MM_2} \tag{P1}, \label{eq:KroneckerProperty1} \\
	\big( \MM_1 \otimes \MM_2 \big)^T  = \big( \MM_1^T \otimes \MM_2^T \big) \tag{P2} \label{eq:KroneckerProperty2}, \\
	\big( \MM_1 \otimes \MM_2 \big) \big( \MM_3 \otimes \MM_4 \big)  = \big( \MM_1 \MM_3 \big) \otimes \big( \MM_2 \MM_4 \big) \tag{P3}. \label{eq:KroneckerProperty3}
\end{gather*}
Let us start the computation of \(  \Hess_{\weightsPointVec} \) by rearranging the loss function so as to simplify the differentiation w.r.t. \( \weightsPointVec_k \). Specifically, we have that
\begin{align}
	\loss(\weightsPointVec) & = \E{\Big\| y-\prod_{j=1}^{m} \weights_j x \Big\|^2} \nonumber\\
	&= \E{\Big\| y-\Big( \prod_{i=k+1}^{m} \weights_i \Big) \weights_k \Big(\prod_{j=1}^{k-1} \weights_j x \Big) \Big\|^2} \nonumber \\
	& = \E{\Big\| y-  \Big(\prod_{j=1}^{k-1} \weights_j x \Big)^T \otimes \Big( \prod_{i=k+1}^{m} \weights_i \Big) \weightsPointVec_k   \Big\|^2}\nonumber \\
	& = \E{\Big\| y-  \Big[x^T\Big(\prod_{j=1}^{k-1} \weights_j \Big)^T\Big] \otimes \Big[\Identity\Big( \prod_{i=k+1}^{m} \weights_i \Big)\Big] \weightsPointVec_k   \Big\|^2}\nonumber \\
	& = \E{\Big\| y-  \Big[ x \otimes \Identity \Big]^T \Big[  \Big(\prod_{j=1}^{k-1} \weights_j \Big)^T \otimes \Big( \prod_{i=k+1}^{m} \weights_i \Big)\Big] \weightsPointVec_k   \Big\|^2},
\end{align}
where in the third equality we used property \eqref{eq:KroneckerProperty1}, and in the last we used properties \eqref{eq:KroneckerProperty2} and \eqref{eq:KroneckerProperty3}. To simplify expressions, we define the following matrices
\begin{equation}
	\UU_k \triangleq  \Big(\prod_{j=1}^{k-1} \weights_j \Big)^T \otimes \Big( \prod_{i=k+1}^{m} \weights_i \Big) \qquad \text{and} \qquad  \XX \triangleq  x \otimes \Identity .
\end{equation}
Thus, the loss function \eqref{eq:VectorLossFuncDef} is given by
\begin{equation}
	\loss(\weightsPointVec) = \E{\Big\| y-  \XX^T \UU_k \weightsPointVec_k \Big\|^2}.
\end{equation}
Now we are ready to calculate the partial derivative of \( \loss(\weightsPointVec) \) w.r.t \(  \weightsPointVec_k \). Notice that \( \UU_k \) is not a function of \( \weightsPointVec_k \), therefore
\begin{equation}
	\frac{\partial}{ \partial \weightsPointVec_k} \E{\Big\| y-  \XX^T \UU_k \weightsPointVec_k \Big\|^2} = -2 \E{ \UU_k^T \XX \big( y - \XX^T \UU_k \weightsPointVec_k \big) } = 2  \UU_k^T \big(\E{\XX \XX^T} \UU_k \weightsPointVec_k -  \E{\XX y} \big).
\end{equation}
Furthermore,
\begin{equation}
	\E{\XX \XX^T} = \E{ \big( x \otimes \Identity \big)\big( x \otimes \Identity \big)^T } = \E{ \big( xx^T \otimes \Identity \big)  } = \big( \E{ xx^T }  \big) \otimes \Identity   = \EmpCovMat{x} \otimes \Identity ,
\end{equation}
where in the second equality we used properties \eqref{eq:KroneckerProperty2} and \eqref{eq:KroneckerProperty3}, and in the third equality we used the linearity of the Kronecker product. Additionally,
\begin{equation}
	\E{\XX y} =  \E{ \big(x \otimes \Identity \big) y} = \E{ \vectorize{y x^T} } = \vectorize{\EmpCovMat{yx}},
\end{equation}
where in the second step we used \eqref{eq:KroneckerProperty1}. Overall we have that
\begin{equation} \label{eq:FirstOrderDevVecCase}
	\frac{\partial}{ \partial \weightsPointVec_k} \loss(\weightsPointVec) = 2  \UU_k^T \left[ \left( \EmpCovMat{x} \otimes \Identity \right)  \UU_k \weightsPointVec_k -  \vectorize{\EmpCovMat{yx}} \right].
\end{equation}
Next we prepare Eq.~\eqref{eq:FirstOrderDevVecCase} for differentiation w.r.t \( \weightsPointVec_q \). First, for all \( k \)
\begin{equation}\label{eq:Tk_Times_w_k}
	\UU_k \weightsPointVec_k = \left[ \Big(\prod_{j=1}^{k-1} \weights_j \Big)^T \otimes \Big( \prod_{i=k+1}^{m} \weights_i \Big) \right] \vectorize{\weights_k}  = \vectorize{ \Big( \prod_{i=k+1}^{m} \weights_i \Big) \weights_k \Big( \prod_{j=1}^{k-1} \weights_j \Big) } = \vectorize{ \prod_{i=1}^{m} \weights_i  }.
\end{equation}
Particularly, this means that the value of the term \( \UU_k \weightsPointVec_k \) is the same for all \( k \). Hence, \( \UU_k \weightsPointVec_k = \UU_q \weightsPointVec_q \) and therefore
\begin{equation}
	\frac{\partial}{ \partial \weightsPointVec_k} \loss(\weightsPointVec) = 2  \UU_k^T \left[ \left( \EmpCovMat{x} \otimes \Identity \right)  \UU_q \weightsPointVec_q -  \vectorize{\EmpCovMat{yx}} \right].
\end{equation}
Now, let us differentiate the vector \( \frac{\partial}{ \partial \weightsPointVec_k} \loss(\weightsPointVec) \) w.r.t. the scalar \( \weightsPointVec_{q,l}  \), which is the \( l \)th element in the vector \( \weightsPointVec_q \). Notice that \( \UU_k \) and \( \weightsPointVec_q \) itself are the only terms which depend on \( \weightsPointVec_q \). Therefore, by the product rule of differentiation using the denominator-layout notation\footnote{Where the derivative \( \frac{\partial { \bf A}}{ \partial z }  \) of a matrix \({ \bf A} \) w.r.t. a scalar \( z \) is laid out according to \({ \bf A}^T \).}
\begin{align} \label{eq:SecondOrderDerVectorTByScalar}
	\scndPrtDev{\weightsPointVec}{q,l}{k} \loss(\weightsPointVec) & = 2  \frac{\partial}{\partial \weightsPointVec_{q,l} } \Big(  \UU_k^T \Big[ \Big( \EmpCovMat{x} \otimes \Identity \Big)  \UU_q \weightsPointVec_q -  \vectorize{\EmpCovMat{yx}} \Big]\Big) \nonumber \\
	& = 2  \Big[ \Big( \EmpCovMat{x} \otimes \Identity \Big)  \UU_q \weightsPointVec_q -  \vectorize{\EmpCovMat{yx}} \Big]^T \left( \frac{\partial}{\partial \weightsPointVec_{q,l} }   \UU_k^T \right) +  2 \left( \frac{\partial}{\partial \weightsPointVec_{q,l} } \weightsPointVec_q \right)  \UU_q^T  \left( \EmpCovMat{x} \otimes \Identity \right)  \UU_k .
\end{align}
However, at a global minimum
\begin{equation}
	\left( \EmpCovMat{x} \otimes \Identity \right)  \UU_q \weightsPointVec_q = \left( \EmpCovMat{x} \otimes \Identity \right)  \vectorize{ \prod_{i=1}^{m} \weights_i} = \vectorize{ \prod_{i=1}^{m} \weights_i \EmpCovMat{x}} = \vectorize{ \EmpCovMat{yx} \EmpCovMat{x}^{-1} \EmpCovMat{x} } = \vectorize{ \EmpCovMat{yx} },
\end{equation}
where in the first equality we used \eqref{eq:Tk_Times_w_k}, in the second equality we used \eqref{eq:KroneckerProperty1} and in the third equality we used the assumption that \( \weightsPointVec \in \minset \). Hence, for all \( 1 \le q \le m \) we have that
\begin{equation}
	\left( \EmpCovMat{x} \otimes \Identity \right)  \UU_q \weightsPointVec_q -  \vectorize{\EmpCovMat{yx}} = \zeroVec.
\end{equation}
Therefore, \eqref{eq:SecondOrderDerVectorTByScalar} is reduced to
\begin{equation}
	\scndPrtDev{\weightsPointVec}{q,l}{k} \loss(\weightsPointVec) = 2  \left( \frac{\partial}{\partial \weightsPointVec_{q,l} } \weightsPointVec_q \right)  \UU_q^T  \left( \EmpCovMat{x} \otimes \Identity \right)  \UU_k ,
\end{equation}
for all \( 1 \le l \le m \). Hence,
\begin{equation}
	\scndPrtDev{\weightsPointVec}{q}{k} \loss(\weightsPointVec) = 2 \left( \frac{\partial}{\partial \weightsPointVec_q }   \weightsPointVec_q \right) \UU_q^T  \left( \EmpCovMat{x} \otimes \Identity \right)  \UU_k  = 2  \UU_q^T  \left( \EmpCovMat{x} \otimes \Identity \right)  \UU_k = 2  \UU_q^T  \left( \EmpCovMat{x}^{\frac{1}{2}} \otimes \Identity \right)^T  \left( \EmpCovMat{x}^{\frac{1}{2}} \otimes \Identity \right) \UU_k,
\end{equation}
where \( \EmpCovMat{x}^{\frac{1}{2}} \) is the symmetric square root matrix of \( \EmpCovMat{x} \). Let us define the matrices \(\{ \bPhi_k \}_{k = 1}^m\) as
\begin{equation}
	\bPhi_k = \UU_k^T  \left( \EmpCovMat{x}^{\frac{1}{2}} \otimes \Identity \right)^T  = \left[ \Bigg(\prod_{j=1}^{k-1} \weights_j \Bigg) \otimes \left( \prod_{i=k+1}^{m} \weights_i \right)^T \right] \left( \EmpCovMat{x}^{\frac{1}{2}} \otimes \Identity \right)  = \Bigg(\prod_{j=1}^{k-1} \weights_j\EmpCovMat{x}^{\frac{1}{2}}  \Bigg) \otimes \left( \prod_{i=k+1}^{m} \weights_i \right)^T.
\end{equation}
Thus,
\begin{equation}
	\frac{\partial^2}{\partial \weightsPointVec_q \partial \weightsPointVec_k} \loss(\weightsPointVec) = 2 \bPhi_q \bPhi_k^T.
\end{equation}
Finally, the Hessian matrix of the loss function \( \loss(\weightsPointVec) \) is given by
\begin{equation}
	\Hess_{\weightsPointVec} = 2\bPhi \bPhi^T,
\end{equation}
where \( \bPhi = [\bPhi_1^T , \bPhi_2^T, \ldots, \bPhi_m^T ]^T \).

\section{The Missing Parts of the Proof of Theorem~\ref{theorem:LamdaMaxMinimalValue}}\label{appendix:LambdaMaxProofParts}

\subsection{Proof of Lemma \ref{lemma:myInequality}}\label{appendix:LowerBoundLemmaProof}

The proof is straightforward. We have
\begin{equation}
\sum_{k = 1}^{m} \| \bPsi_k \|_{\rm F}^2 \geq  \sum_{k = 1}^{m} \| \bPsi_k \|_{2}^2 \geq m \Bigg[ \prod_{k = 1}^{m} \| \bPsi_k \|_{2} \Bigg]^{\frac{2}{m}} \geq  m \Bigg[ \Big\| \prod_{k = 1}^{m} \bPsi_k \Big\|_{2} \Bigg]^{\frac{2}{m}},
\end{equation}
where in the first inequality we used the fact that \( \| \bPsi\|_{\rm F} \geq \| \bPsi \|_{2} \) for any matrix \( \bPsi \in \R^{d_1 \times d_2} \). The second inequality is due to the inequality of arithmetic and geometric means. In the final inequality we used the fact that \( \| \cdot \|_{2} \) is a sub-multiplicative matrix norm, meaning \( \| \bPsi \|_2 \| \boldsymbol{\Phi} \|_2 \geq  \| \bPsi \boldsymbol{\Phi} \|_2   \) for any pair of matrices \( \bPsi \in \R^{d_1 \times d_2} , \boldsymbol{\Phi} \in \R^{d_2 \times d_3}\).

\subsection{Maximal Value of \( \nu \)}\label{appendix:MaxValueLowerBound}
On the one hand, for any \( \BB \in \R^{d_y \times d_x} \) such that \( \| \BB  \|_{\rm F} = 1 \), 
\begin{equation} 
\Big\| \big(  \BB \optmlOper^T \big)^{m-1} \BB \Big\|_2 \leq
\big\|  \BB \big\|_2^{m} \big\| \optmlOper \big\|_2^{m-1} \leq (\sigma_{\max}( \optmlOper))^{m-1},
\end{equation}
where in the second inequality we used \( \| \BB \|_2 \leq \|  \BB \|_{\rm F} = 1 \). On the other hand, this upper bound is achieved by \( \BB = \uu \vv^T \), as
\begin{align} \label{eq:TermWithSingularVectors}
\Big\| \big(  \uu \vv^T \optmlOper^T \big)^{m-1} \uu \vv^T \Big\|_2 & = \Big\|   \uu \big( \vv^T \optmlOper^T \uu \big)^{m-1}  \vv^T \Big\|_2 = \big( \vv^T \optmlOper^T \uu \big)^{m-1} \big\| \uu \vv^T \big\|_2 \nonumber \\
& = (\sigma_{\max}( \optmlOper))^{m-1}  \| \uu \| \| \vv \| = (\sigma_{\max}( \optmlOper))^{m-1}.
\end{align}
Therefore, 
\begin{equation}
\max_{ \|  \BB  \|_{\rm F} =1 } \nu(\BB) = 2m \times (\sigma_{\max} ( \optmlOper ))^{2(1-\frac{1}{m})}.
\end{equation}

\subsection{Maximal Eigenvalue at the Canonical Solution~\eqref{eq:CanonicalSolution}}\label{appendix:MaxEigenvalCanonicalSol}
In \eqref{eq:eigenvalueVectorForm} we have
\begin{equation}\label{eq:LambdaMax2PhiB}
	\lambda_{\max}\big( \Hess_{\weightsPointVec} \big) = \max_{ \| \bb \| = 1 } 2 \| \bPhi \bb \|^2 .
\end{equation}
Note that $\| \bPhi \bb \|^2 = \sum_{k=1}^m \|\bPhi_k \bb\|^2$. Using the definition of $ \bPhi_k $ in \eqref{eq:PhiDef} we get 
\begin{equation}
\norm{\bPhi_k \bb}^2 = \bigg\|{\bigg(\prod_{j=1}^{k-1} \weights_j\EmpCovMat{x}^{\frac{1}{2}}  \bigg) \otimes \bigg( \prod_{i=k+1}^{m} \weights_i \bigg)^T \bb}\bigg\|^2 = \bigg\|{\bigg( \prod_{i=k+1}^{m} \weights_i \bigg)^T \BB \bigg(\prod_{j=1}^{k-1} \weights_j\EmpCovMat{x}^{\frac{1}{2}}  \bigg)^T}\bigg\|^2_{\rm F},
\end{equation}
where we used \eqref{eq:KroneckerProperty1} with \( \bb = \vectorize{\BB} \). Therefore,
\begin{equation}\label{eq:LambdaMaxSum}
\lambda_{\max}\big( \Hess_{\weightsPointVec} \big) = \max_{ \| \BB \|_{\rm F} = 1 } 2 \sum_{k = 1}^{m} \Big\| {\Big( \prod_{i=k+1}^{m}\!\!\! \weights_i \Big)^T \BB  \EmpCovMat{x}^{\frac{1}{2}} \Big(\prod_{j=1}^{k-1} \weights_j  \Big)^T }  \Big\|_{\rm F}^2.
\end{equation}
Substituting the canonical solution \eqref{eq:CanonicalSolution} in this optimization problem, we obtain
\begin{equation} \label{eq:CanonicalSolStart}
	\sum_{k = 1}^{m} \Big\| {\Big( \prod_{i=k+1}^{m} \weights_i \Big)^T \BB \Big(\prod_{j=1}^{k-1} \weights_j  \Big)^T }  \Big\|_{\rm F}^2 
	=\sum_{k = 1}^{m} \Big\| {\Big( \prod_{i=k+1}^{m} \SSS^{\frac{1}{m}}_i \Big)^T  \UU^T \BB \VV \Big(\prod_{j=1}^{k-1} \SSS^{\frac{1}{m}}_j  \Big)^T }  \Big\|_{\rm F}^2,
\end{equation}
where in the first and the last terms of the series (\(k=1\) and \(k=m\)) we used the fact that \( \UU \) and \( \VV \) are unitary matrices, so that
\begin{equation}
	\Big\| \Big( \prod_{i=2}^{m} \SSS^{\frac{1}{m}}_i  \Big)^T  \UU^T \BB \Big\|_{\rm F}^2 = \Big\| \Big( \prod_{i=2}^{m} \SSS^{\frac{1}{m}}_i  \Big)^T \UU^T \BB \VV \Big\|_{\rm F}^2,
	\qquad \Big\| \BB \VV  \Big( \prod_{j=1}^{m-1} \SSS^{\frac{1}{m}}_j \Big)^T  \Big\|_{\rm F}^2 = \Big\|\UU^T  \BB \VV \Big( \prod_{j=1}^{m-1} \SSS^{\frac{1}{m}}_j \Big)^T \Big\|_{\rm F}^2 .
\end{equation}
Note that $ \prod_{i=k+1}^{m} \SSS^{\frac{1}{m}}_i  $ is a diagonal $ d_y \times d_{k} $ matrix, whose $q$th diagonal entry is $(\sigma_q(\optmlOper))^{(m-k)/m}$ (where \( \sigma_{q}(\optmlOper) \) is the \(q\)th largest singular value of \( \optmlOper \)). Similarly, $ \prod_{j=1}^{k-1} \SSS^{\frac{1}{m}}_j  $ is a diagonal $ d_{k-1}\times d_x $ matrix, whose $q$th diagonal entry is $(\sigma_q(\optmlOper))^{(k-1)/m}$. Therefore, we can write
\begin{equation}
	\sum_{k = 1}^{m} \Big\| {\Big( \prod_{i=k+1}^{m} \SSS^{\frac{1}{m}}_i \Big)^T  \UU^T \BB \VV \Big(\prod_{j=1}^{k-1} \SSS^{\frac{1}{m}}_j  \Big)^T }  \Big\|_{\rm F}^2 
	= \sum_{k = 1}^{m} \Big\| \Big( \SSS^{\frac{m-k}{m}} \Big)^T \UU^T \BB \VV \Big( \SSS^{\frac{k-1}{m}} \Big)^T \Big\|_{\rm F}^2,
\end{equation}
where $\SSS^\alpha$ denotes a $d_y\times d_x$ diagonal matrix whose $q$th diagonal entry is $(\sigma_q(\optmlOper))^\alpha$. Here, we used the fact that the Frobenius norm is unaffected by zero entries, and thus removed/added zero rows/columns.

Next, we preform the change of variables \( \tilde{\BB} = \UU^T \BB \VV \in \R^{d_y \times d_x} \) to obtain the following optimization problem
\begin{equation}
\max_{\tilde{\BB} \in \R^{d \times d} }  2 \sum_{k = 1}^{m} \Big\| \Big( \SSS^{\frac{m-k}{m}} \Big)^T \tilde{\BB} \Big( \SSS^{\frac{k-1}{m}} \Big)^T  \Big\|_{\rm F}^2 \qquad \text{s.t.} \qquad \big\| \tilde{\BB}  \big\|^2_{\rm F} = 1 .
\end{equation}
Writing the objective in terms of the elements of \( \tilde{\BB} \), which we denote by \( \{\tilde{b}_{i,j} \} \), gives
\begin{equation}
2 \sum_{k = 1}^{m} \Big\| \Big( \SSS^{\frac{m-k}{m}}  \Big)^T \tilde{\BB} \Big( \SSS^{\frac{k-1}{m}} \Big)^T  \Big\|_{\rm F}^2 = 2 \sum_{k = 1}^{m} \sum_{i = 1}^{d} \sum_{j = 1}^{d} \Big[ (\sigma_{j}( \optmlOper) )^{\frac{m-k}{m}} (\sigma_{i}( \optmlOper )) ^{\frac{k-1}{m}}\tilde{b} _{i,j} \Big]^2,
\end{equation}
where \( d = \min\{d_x,d_y\}  \) is the number of singular values of \( \optmlOper \). By changing the order of the summation, we get
\begin{equation}
\max_{\tilde{b}_{1,1}, \ldots, \tilde{b}_{d,d} \in \R } \  2 \sum_{i,j = 1}^{d}\tilde{b}^2 _{i,j} \sum_{k = 1}^{m}  \Big[ (\sigma_{j}( \optmlOper )) ^{\frac{m-k}{m}} (\sigma_{i}( \optmlOper)) ^{\frac{k-1}{m}}\Big]^2 \qquad \text{s.t.} \qquad \sum_{i,j = 1}^{d}\tilde{b}^2 _{i,j}  = 1 .
\end{equation}
This is a simple linear optimization problem over the unit simplex, whose optimal value is attained at one of the vertices,
\begin{equation} \label{eq:CanonicalSolEnd}
\max_{ i,j \in \{ 1,\ldots,d \} } \  2 \sum_{k = 1}^{m}  \Big[ (\sigma_{j}( \optmlOper )) ^{\frac{m-k}{m}} (\sigma_{i}( \optmlOper )) ^{\frac{k-1}{m}}\Big]^2.
\end{equation}
The maximal value is attained for \( i = j =1 \), thus the value of \eqref{eq:eigenvalueMatrixForm} for the canonical solution is
\begin{equation}
2 \sum_{k = 1}^{m}  \Big[( \sigma_{1}( \optmlOper )) ^{\frac{m-k}{m}} (\sigma_{1}( \optmlOper )) ^{\frac{k-1}{m}}\Big]^2 = 2m \times (\sigma_{\max}( \optmlOper))^{2(1-\frac{1}{m})}.
\end{equation}
This result shows that the canonical solution \eqref{eq:CanonicalSolution} is indeed a minimizer of the maximal eigenvalue of the Hessian matrix.

\subsection{Proof of the Top Eigenvector of \( \Hess_{\weightsPointVec} \)}\label{appendix:TopEigenvectorProof}
On the one hand, according to Section \ref{sec:Derivation}, for any flattest minimum point \( \weightsPointVec \in \wideminset \), the largest eigenvalue satisfies
\begin{equation}\label{eq:LambdaMaxValue}
	\lambda_{\max}(\Hess_{\weightsPointVec}) = 2m \times (\sigma_{\max} ( \optmlOper )) ^{2(1-\frac{1}{m})}.
\end{equation}
On the other hand, the maximal eigenvalue of the Hessian matrix is the solution to the optimization problem \eqref{eq:LambdaMaxSum}, in which \( \bb = \vectorize{\BB} \) is the eigenvector of $\HessTag_{\weightsPointVec}$ (see \eqref{eq:LambdaMax2PhiB}). Substituting \( \BB^* = \uu \vv^T \) (\ie \(  \bb^* = \vv \otimes  \uu \)) in the objective function, we get
\begin{align} \label{eq:LowerBoundOnQuadForm}
	2 \sum_{k = 1}^{m} \Big\| {\Big( \prod_{i=k+1}^{m} \weights_i \Big)^T \BB^* \Big(\prod_{j=1}^{k-1} \weights_j \Big)^T }  \Big\|_{\rm F}^2
	& \geq  2m \times \left\| \Bigg[  \BB^* \Big(\prod_{i=1}^{m} \weights_i  \Big)^T \Bigg]^{m-1} \BB^* \right\|_2^{\frac{2}{m}}\nonumber\\
	&  =  2m \times \left\| \big(  \BB^* \optmlOper^T \big)^{m-1} \BB^* \right\|_2^{\frac{2}{m}} \nonumber\\
	&=  2m \times (\sigma_{\max}( \optmlOper))^{2(1-\frac{1}{m})},
\end{align}
where in the second inequality we used Lemma \ref{lemma:myInequality} and explicitly unrolled the product, as in \eqref{eq:MultiplicationUnrolling}, and in the last step we used~\eqref{eq:TermWithSingularVectors}. This proves that \(  \bb^* = \vv \otimes  \uu \) is an eigenvector of \( \hat{\Hess}_{\weightsPointVec} \) corresponding to the maximal eigenvalue. Now, since $\HessTag_{\weightsPointVec} = 2 \bPhi^T \bPhi$ and $\Hess_{\weightsPointVec} = 2 \bPhi \bPhi^T $, we have that $ \bPhi \bb^*=\bPhi ( \vv \otimes  \uu  )$ is the eigenvector of $\Hess_{\weightsPointVec}$ corresponding to its maximal eigenvalue.

\section{Proof of Theorem \ref{theorem:PartialProducts}}\label{appendix:PartialProductsProof}
Let us start the proof by presenting two lemmas.
\begin{lemma}[]\label{lemma:Balance}
	Let \( \EmpCovMat{x} = \Identity \). If \( \weightsPointVec \in \wideminset  \) then for all \( k \in \{ 1,2,\ldots,m \} \)
	\begin{equation} \label{eq:Balance}
	\Big\| \uu^T \prod_{i=k+1}^{m} \weights_i \Big\| 	\Big\| \prod_{j=1}^{k-1} \weights_j \vv \Big\| = (\sigma_{\max} ( \optmlOper ))^{1-\frac{1}{m}}.
	\end{equation}
\end{lemma}
\begin{proof}
	First, observe that for \( \BB^* = \uu \vv^T \), the left-hand side of \eqref{eq:Balance} can be written as
	\begin{equation}\label{eq:BalancedPartialProd}
		\Big\| \uu^T \prod_{i=k+1}^{m} \weights_i \Big\| \Big\| \prod_{j=1}^{k-1} \ \weights_j \ \vv \Big\| = \Big\| {\Big( \prod_{i=k+1}^{m} \weights_i \Big)^T \BB^* \Big(\prod_{j=1}^{k-1} \weights_j \Big)^T }  \Big\|_2.
	\end{equation}
	Now, from Theorem \ref{theorem:LamdaMaxMinimalValue} and Eq.~\eqref{eq:eigenvalueMatrixForm} we can conclude that
	\begin{equation}\label{eq:lambdaMaxBstar1}
	\lambda_{\max}(\Hess_{\weightsPointVec})  =  2 \sum_{k = 1}^{m} \Big\| {\Big( \prod_{i=k+1}^{m} \weights_i \Big)^T \BB^* \Big(\prod_{j=1}^{k-1} \weights_j \Big)^T }  \Big\|_2^2  =  2m \times (\sigma_{\max} ( \optmlOper )) ^{2(1-\frac{1}{m})}.
	\end{equation}
	Note that since \( \BB^* \) is a rank-1 matrix, the entire expression within the norm is rank-1, which is the reason we could replace the Frobenius norm appearing in~\eqref{eq:eigenvalueMatrixForm} by the operator norm (the two norms coincide for rank-1 matrices). Furthermore, by the inequality of arithmetic and geometric means
	\begin{align}\label{eq:lambdaMaxBstarLowerBound}
	2 \sum_{k = 1}^{m} \Big\| {\Big( \prod_{i=k+1}^{m} \weights_i \Big)^T \BB^* \Big(\prod_{j=1}^{k-1} \weights_j \Big)^T }  \Big\|_2^2
	& \geq 2 m \left[ \prod_{k = 1}^{m} \Big\| {\Big( \prod_{i=k+1}^{m} \weights_i \Big)^T \BB^* \Big(\prod_{j=1}^{k-1} \weights_j \Big)^T }  \Big\|_2^2\right]^{\frac{1}{m}} \nonumber\\
	& \geq 2 m \left[ \Big\| \prod_{k = 1}^{m}  {\Big( \prod_{i=k+1}^{m} \weights_i \Big)^T \BB^* \Big(\prod_{j=1}^{k-1} \weights_j \Big)^T }  \Big\|_2^2\right]^{\frac{1}{m}}  \nonumber\\
	& =  2m \times (\sigma_{\max} ( \optmlOper )) ^{2(1-\frac{1}{m})} ,
	\end{align}
	where the second inequality is due to the sub-multiplicativity property of the operator norm, and in the last step we unrolled the product, as in \eqref{eq:MultiplicationUnrolling}, and used \eqref{eq:TermWithSingularVectors}. From \eqref{eq:lambdaMaxBstar1} and \eqref{eq:lambdaMaxBstarLowerBound} we obtain that the inequality of arithmetic and geometric means in \eqref{eq:lambdaMaxBstarLowerBound} is achieved with equality. This happens if and only if all summands in the series are equal. Thus, we conclude that for all \( k \in \{ 1,2,\ldots,m \} \),
	\begin{equation}
		\Big\| {\Big( \prod_{i=k+1}^{m} \weights_i \Big)^T \BB^* \Big(\prod_{j=1}^{k-1} \weights_j \Big)^T }  \Big\|_2 = (\sigma_{\max}( \optmlOper)) ^{1-\frac{1}{m}},
	\end{equation}
	and together with \eqref{eq:BalancedPartialProd}, this implies that
	\begin{equation}
	\Big\| \uu^T \prod_{i=k+1}^{m} \weights_i \Big\| \Big\| \prod_{j=1}^{k-1} \ \weights_j \ \vv \Big\| = (\sigma_{\max}( \optmlOper)) ^{1-\frac{1}{m}}.
	\end{equation}
\end{proof}

While Lemma~\ref{lemma:Balance} characterizes the norms of the vectors $ \uu^T \prod_{i=k+1}^{m} \weights_i $ and $\prod_{j=1}^{k-1} \weights_j \vv $, the next Lemma characterizes their directions.
\begin{lemma}[]\label{lemma:DirectionMatching}
	Let \( \EmpCovMat{x} = \Identity \). If \( \weightsPointVec \in \wideminset \) then for all \( k \in \{0,1,2,\ldots,m \} \)
	\begin{equation} \label{eq:DirectionMatching}
	\frac{1}{\Big\|  \Big(   \prod_{i=k+1}^{m} \weights_i \Big)^T \uu \Big\|}  \Bigg(   \prod_{i=k+1}^{m} \weights_i \Bigg)^T \uu
	= \frac{1}{\Big\| \prod_{j=1}^{k} \ \weights_j \ \vv \Big\|}
	\prod_{j=1}^{k} \ \weights_j \ \vv.
	\end{equation}
\end{lemma}

\begin{proof}
	From Lemma \ref{lemma:Balance} we have
	\begin{equation} \label{eq:ProofDirectionMatching1}
		\prod_{k = 1}^{m} \Big\| {\Big( \prod_{i=k+1}^{m} \weights_i \Big)^T \uu\Big\|   \Big\| \vv^T \Big(\prod_{j=1}^{k-1} \weights_j \Big)^T }  \Big\| =  (\sigma_{\max} ( \optmlOper )) ^{m-1} .
	\end{equation}
	Recall that in our convention (see Section \ref{sec:MainResults}), for $ k = 1,m $ we have
	\begin{equation}
		\Big\| \vv^T \Big(\prod_{j=1}^{0} \weights_j \Big)^T \Big\| = \| \vv^T \| = 1, \qquad  \Big\|  \Big( \prod_{i=m+1}^{m} \weights_i \Big)^T \uu \Big\| = \| \uu \| = 1.
	\end{equation}
	Therefore, \eqref{eq:ProofDirectionMatching1} can be written as
	\begin{equation} \label{eq:ProofDirectionMatching2}
		\prod_{k = 1}^{m-1}  \Big\| \vv^T \Big(\prod_{j=1}^{k} \weights_j \Big)^T  \Big\| \Big\| \Big( \prod_{i=k+1}^{m} \weights_i \Big)^T \uu\Big\|  = ( \sigma_{\max} ( \optmlOper )) ^{m-1}.
	\end{equation}
	On the other hand, by the Cauchy–Schwarz inequality we have
	\begin{align} \label{eq:ProofDirectionMatching3}
		\prod_{k = 1}^{m-1}  \Big\| \vv^T \Big(\prod_{j=1}^{k} \weights_j \Big)^T  \Big\| \Big\| \Big( \prod_{i=k+1}^{m} \weights_i \Big)^T \uu\Big\| &\geq  
		\prod_{k = 1}^{m-1} \vv^T \Big(\prod_{j=1}^{k} \weights_j \Big)^T  \Big(  \prod_{i=k+1}^{m} \weights_i \Big)^T \uu \nonumber \\ 
		 & =\prod_{k = 1}^{m-1} \vv^T  \optmlOper ^T \uu \nonumber\\
		& = (\sigma_{\max} ( \optmlOper )) ^{(m-1)}.
	\end{align}
	From \eqref{eq:ProofDirectionMatching2} we have that the Cauchy–Schwartz inequalities are achieved with equality. Thus, for all \( k \in \{0,1,2,\ldots,m \} \)
	\begin{equation}
		\frac{1}{\Big\|  \Big( \prod_{i=k+1}^{m} \weights_i \Big)^T \uu \Big\|}  \Bigg( \prod_{i=k+1}^{m} \weights_i \Bigg)^T \uu
		= \frac{1}{\Big\| \prod_{j=1}^{k} \ \weights_j \ \vv \Big\|}
		\prod_{j=1}^{k} \ \weights_j \ \vv .
	\end{equation}
	
\end{proof}

Now we are ready to prove Theorem \ref{theorem:PartialProducts}. From Lemma \ref{lemma:DirectionMatching}, we have
\begin{equation} \label{eq:CummProdSingularProof1}
	\frac{1}{\Big\| \prod_{j=1}^{k} \ \weights_j \ \vv \Big\|} \prod_{j=1}^{k} \ \weights_j \ \vv
	= \frac{1}{\Big\|  \Big( \prod_{i=k+1}^{m} \weights_i \Big)^T \uu \Big\|}  \Bigg(   \prod_{i=k+1}^{m} \weights_i \Bigg)^T \uu.
\end{equation}
Multiplying by $ ( \prod_{j=1}^{k} \ \weights_j )^T $ from the left, \eqref{eq:CummProdSingularProof1} becomes
\begin{equation} \label{eq:CummProdSingularProof2}
	\frac{1}{\Big\| \prod_{j=1}^{k} \ \weights_j \ \vv \Big\|} \Bigg( \prod_{j=1}^{k} \ \weights_j \Bigg)^T
	\prod_{j=1}^{k} \ \weights_j \ \vv = \frac{1}{\Big\|  \Big(   \prod_{i=k+1}^{m} \weights_i \Big)^T \uu \Big\|}  \optmlOper ^T \uu.
\end{equation}
Note that $ \optmlOper^T \uu = \sigma_{\max}(\optmlOper) \vv $. Therefore,
\begin{equation} \label{eq:CummProdSingularProof3}
	\Bigg[\Bigg( \prod_{j=1}^{k} \ \weights_j \Bigg)^T \prod_{j=1}^{k} \ \weights_j \Bigg] \ \vv = 
	\frac{\Big\| \prod_{j=1}^{k} \ \weights_j \ \vv \Big\|}{\Big\|  \Big(   \prod_{i=k+1}^{m} \weights_i \Big)^T \uu \Big\|} \sigma_{\max} \big( \optmlOper \big) \ \vv.
\end{equation}
This shows that $ \vv $ is an eigenvector of $ \big( \prod_{j=1}^{k} \weights_j \big)^T \prod_{j=1}^{k} \weights_j $, \ie a singular vector of $ \prod_{j=1}^{k} \weights_j $. To compute the corresponding singular value, let us multiply this equation by $ \vv^T $ form the left to get the following result.
\begin{equation} \label{eq:CummProdSingularProof4}
	\Big\| \prod_{j=1}^{k} \weights_j \ \vv \Big\| \ \Big\| \Big( \prod_{i=k+1}^{m} \weights_i \Big)^T \uu \Big\| = \sigma_{\max} \big( \optmlOper \big).
\end{equation}
Recall from Lemma \ref{lemma:Balance} that $\| \uu^T \prod_{i=k+1}^{m} \weights_i \| = (\sigma_{\max} ( \optmlOper ))^{1-1/m} / \| \prod_{j=1}^{k-1} \weights_j \vv \|$. Substituting into \eqref{eq:CummProdSingularProof4}, we obtain that
\begin{equation} \label{eq:CummProdSingularProof5}
	\Big\| \prod_{j=1}^{k} \ \weights_j \ \vv \Big\| = (\sigma_{\max} \big( \optmlOper \big))^{\frac{1}{m}} \times \Big\| \prod_{j=1}^{k-1} \ \weights_j \ \vv \Big\|.
\end{equation}
By unwrapping this recursive formula with an initial condition for $k=0$ of $ \big\| \prod_{j=1}^{0} \weights_j \vv \big\| = \| \vv \| = 1 $, we get 
\begin{equation} \label{eq:CummProdSingularProof6}
	\Big\| \prod_{j=1}^{k} \weights_j  \vv \Big\| = \sigma_{\max}(\optmlOper)^{\frac{k}{m}}.
\end{equation}
The proof for the left singular vector and its corresponding singular value is the same.

Next, we prove the bound on the intermediate gain. By Theorem \ref{theorem:LamdaMaxMinimalValue} and Eq.~\eqref{eq:eigenvalueMatrixForm},
\begin{equation}\label{eq:SingularValueBound1}
\max_{\| \BB \|_{\rm F} = 1  } \sum_{l = 1}^{m}  \Big\| {\Big( \prod_{i=l+1}^{m} \weights_i \Big)^T \BB \Big(\prod_{j=1}^{l-1} \weights_j \Big)^T }  \Big\|_{\rm F}^2 = m \times (\sigma_{\max} (\optmlOper))^{2(1-\frac{1}{m})}.
\end{equation}
Now, for any \(k\), we have that
\begin{equation}\label{eq:SingularValueBound2}
\max_{\| \BB \|_{\rm F} = 1  } \sum_{l = 1}^{m}  \Big\| {\Big( \prod_{i=l+1}^{m} \weights_i \Big)^T \BB \Big(\prod_{j=1}^{l-1} \weights_j \Big)^T }  \Big\|_{\rm F}^2 \geq \max_{\| \BB \|_{\rm F} = 1  } \Big\| {\Big( \prod_{i=k+1}^{m} \weights_i \Big)^T \BB \Big(\prod_{j=1}^{k-1} \weights_j \Big)^T }  \Big\|_{\rm F}^2  .
\end{equation}
Furthermore, note that
\begin{align}\label{eq:SingularValueBound3}
	\max_{\| \BB \|_{\rm F} = 1  } \Big\| {\Big( \prod_{i=k+1}^{m} \weights_i \Big)^T \BB \Big(\prod_{j=1}^{k-1} \weights_j \Big)^T }  \Big\|_{\rm F}
	& = \max_{\| b \| = 1  } \Big\| \Big(\prod_{j=1}^{k-1} \weights_j \Big) \otimes  \Big( \prod_{i=k+1}^{m} \weights_i \Big)^T   \bb \Big\| \nonumber\\
	& = \sigma_{\max}  \left( \Big(\prod_{j=1}^{k-1} \weights_j \Big) \otimes  \Big( \prod_{i=k+1}^{m} \weights_i \Big)^T \right) \nonumber\\
	& = \sigma_{\max}\Bigg( \prod_{i=1}^{k-1} \weights_i  \Bigg) \times \sigma_{\max} \Bigg(  \prod_{i=k+1}^{m} \weights_i  \Bigg).
\end{align}
Therefore, this implies that
\begin{equation}\label{eq:SingularValueBound4}
\sigma_{\max}\Bigg( \prod_{i=1}^{k-1} \weights_i  \Bigg) \times \sigma_{\max} \Bigg( \prod_{i=k+1}^{m} \weights_i \Bigg) \leq \sqrt{m} \times (\sigma_{\max} ( \optmlOper)) ^{(1-\frac{1}{m})},
\end{equation}
or, equivalently, that
\begin{equation}\label{eq:SingularValueBound4a}
\sigma_{\max}\Bigg( \prod_{i=1}^{k-1} \weights_i  \Bigg)  \leq \frac{\sqrt{m} \times (\sigma_{\max} ( \optmlOper )) ^{(1-\frac{1}{m})}}{ \sigma_{\max} \big( \prod_{i=k+1}^{m} \weights_i \big)}.
\end{equation}
By the first part of Theorem \ref{theorem:PartialProducts} we have
\begin{equation}\label{eq:SingularValueBound5}
\sigma_{\max} \Bigg(  \prod_{i=k+1}^{m} \weights_i  \Bigg) \geq  \Bigg(  \prod_{i=k+1}^{m} \weights_i \uu \Bigg) = (\sigma_{\max} ( \optmlOper )) ^{(1-\frac{k}{m})}.
\end{equation}
Hence, we can further bound \eqref{eq:SingularValueBound4a} from above as 
\begin{equation}\label{eq:SingularValueBound6}
\sigma_{\max}\Bigg( \prod_{i=1}^{k-1} \weights_i  \Bigg) \leq \frac{\sqrt{m} \times (\sigma_{\max} ( \optmlOper )) ^{(1-\frac{1}{m})}	}{\sigma_{\max} \big(  \prod_{i=k+1}^{m} \weights_i  \big) } \leq \sqrt{m} \times (\sigma_{\max} ( \optmlOper )) ^{\frac{k-1}{m}}	.
\end{equation}
The proof of the other direction is similar.

\section{Proof of Theorem \ref{theorem:SingularVectorsAndValues}}\label{appendix:SingularVectorsAndValuesProof}
By Lemma \ref{lemma:DirectionMatching}
\begin{equation} \label{eq:SingularVectorsAndValues1}
	\frac{1}{\Big\|  \Big( \prod_{i=k+1}^{m} \weights_i \Big)^T \uu \Big\|}  \Bigg( \prod_{i=k+1}^{m} \weights_i \Bigg)^T \uu
	= \frac{1}{\Big\| \prod_{j=1}^{k} \weights_j  \vv \Big\|}
	\prod_{j=1}^{k} \weights_j \vv.
\end{equation}
Multiplying both sides by \( \weights_k^T \) from the left, we obtain
\begin{equation} \label{eq:SingularVectorsAndValues2}
	\frac{1}{\Big\|  \Big( \prod_{i=k+1}^{m} \weights_i \Big)^T \uu \Big\|}  \Bigg( \prod_{i=k}^{m} \weights_i \Bigg)^T \uu
	= \frac{1}{\Big\| \prod_{j=1}^{k}  \weights_j  \vv \Big\|}
	\bigg( \weights_k^T \weights_k \bigg) \prod_{j=1}^{k-1} \weights_j \vv.
\end{equation}
Using Lemma \ref{lemma:DirectionMatching} again for \( k-1 \) we have
\begin{equation} \label{eq:SingularVectorsAndValues3}
\Bigg( \prod_{i=k}^{m} \weights_i \Bigg)^T \uu
= \frac{\Big\|  \Big( \prod_{i=k}^{m} \weights_i \Big)^T \uu \Big\|}{\Big\| \prod_{j=1}^{k-1} \weights_j  \vv \Big\|}
\prod_{j=1}^{k-1} \weights_j \vv.
\end{equation}
Plugging this equation in \eqref{eq:SingularVectorsAndValues2} we get
\begin{equation} \label{eq:SingularVectorsAndValues4}
\frac{\Big\|  \Big( \prod_{i=k}^{m} \weights_i \Big)^T \uu \Big\|}{\Big\| \prod_{j=1}^{k-1} \weights_j  \vv \Big\| \Big\|  \Big( \prod_{i=k+1}^{m} \weights_i \Big)^T \uu \Big\| } \prod_{j=1}^{k-1} \weights_j \vv
= \frac{1}{\Big\| \prod_{j=1}^{k}  \weights_j  \vv \Big\|}
\bigg( \weights_k^T \weights_k \bigg) \prod_{j=1}^{k-1} \weights_j \vv.
\end{equation}
From Theorem \ref{theorem:PartialProducts}, we know that \( \| ( \prod_{i=k}^{m} \weights_i )^T \uu \| = (\sigma_{\max} ( \optmlOper ))^{(m-k+1)/m}  \), \( \| \prod_{j=1}^{k}  \weights_j  \vv \| = (\sigma_{\max} ( \optmlOper ))^{k/m} \), and \( \| \prod_{j=1}^{k-1} \weights_j  \vv \| \|  ( \prod_{i=k+1}^{m} \weights_i )^T \uu \| = (\sigma_{\max} ( \optmlOper ))^{1-1/m} \) . Therefore, \eqref{eq:SingularVectorsAndValues4} can be reduced to
\begin{equation} \label{eq:SingularVectorsAndValues5}
(\sigma_{\max} ( \optmlOper ))^{\frac{2}{m}} \times \rr_k = \big( \weights_k^T \weights_k \big) \rr_k.
\end{equation}
Hence, $ \rr_k $ is an eigenvector of $ \weights_k^T \weights_k $ with a corresponding eigenvalue of \( (\sigma_{\max} ( \optmlOper ))^{2/m} \). Namely, \( \rr_k/ \| \rr_k \| \) is a singular vector of $ \weights_k $ with a corresponding singular value of \( (\sigma_{\max}( \optmlOper))^{1/m} \). Using Lemma \ref{lemma:DirectionMatching} we have
\begin{equation} 
\frac{1}{ \| \rr_k \| } \weights_k \rr_k =(\sigma_{\max}( \optmlOper))^{\frac{1}{m}}\times \frac{1}{ \| \qq_k \| } \qq_k .
\end{equation}
From this equation we deduce that \( \bar{\rr}_k \) and \( \bar{\qq}_k \) are pair of singular vectors of \( \weights_k \), with a singular value of \( (\sigma_{\max}( \optmlOper))^{1/m} \). Note that the equality \( \bar{\rr}_{k+1} = \bar{\qq}_k \) is in fact the result of Lemma \ref{lemma:DirectionMatching}.

\section{Proof of Theorem \ref{theorem:SufficientCondition}}\label{appendix:SufficientConditionProof}
On the one hand, according to Theorem \ref{theorem:LamdaMaxMinimalValue}
\begin{equation}\label{eq:SufficientCondition1}
\min_{\tilde{\weightsPointVec} \in \minset } \lambda_{\max}(\Hess_{\tilde{\weightsPointVec}}) = 2m \times \sigma_{\max} \big( \optmlOper \big) ^{2(1-\frac{1}{m})}.
\end{equation}
On the other hand, given an arbitrary minimum point \( \weightsPointVec \in \minset \)
\begin{align}\label{eq:SufficientCondition2}
\min_{\tilde{\weightsPointVec} \in \minset } \lambda_{\max}(\Hess_{\tilde{\weightsPointVec}}) & \leq \lambda_{\max}(\Hess_{\weightsPointVec}) \nonumber \\
& = \max_{\BB \in \R^{d_y \times d_x} } 2 \sum_{k = 1}^{m}  \Big\| {\Big( \prod_{i=k+1}^{m} \weights_i \Big)^T \BB \Big(\prod_{j=1}^{k-1} \weights_j \Big)^T }  \Big\|_{\rm F}^2 \qquad \text{s.t.} \qquad \| \BB \|_{\rm F} = 1  \\
& \leq \max_{\BB_1,\ldots,\BB_m \in \R^{d_y \times d_x} } 2 \sum_{k = 1}^{m}  \Big\| {\Big( \prod_{i=k+1}^{m} \weights_i \Big)^T \BB_k \Big(\prod_{j=1}^{k-1} \weights_j \Big)^T }  \Big\|_{\rm F}^2 \quad \text{s.t.} \quad \| \BB_1 \|_{\rm F} = \cdots = \| \BB_m \|_{\rm F} = 1  \nonumber.
\end{align}
Here we obtain a separable optimization problem. Let us examine one term from the series
\begin{equation}\label{eq:SufficientCondition3}
\max_{\BB_k \in \R^{d_y \times d_x} } \Big\| {\Big( \prod_{i=k+1}^{m} \weights_i \Big)^T \BB_k \Big(\prod_{j=1}^{k-1} \weights_j \Big)^T }  \Big\|_{\rm F}^2 \qquad \text{s.t.} \qquad \| \BB_k \|_{\rm F} = 1 .
\end{equation}
In \eqref{eq:SingularValueBound3}, we saw that the value of \eqref{eq:SufficientCondition3} is
\begin{equation}\label{eq:SufficientCondition4}
\Big(\sigma_{\max}\Big( \prod_{i=k+1}^{m} \weights_i \Big)\Big)^2 \times \Big(\sigma_{\max}\Big( \prod_{i=1}^{k-1} \weights_i \Big)\Big)^2 \leq \prod_{i \neq k} (\sigma_{\max}(\weights_i))^2,
\end{equation}
where we used the sub-multiplicity property of the operator norm (the top singular value). If \( \sigma_{\max}(\weights_k) = (\sigma_{\max}( \optmlOper))^{1/m} \) for all \( k \), then
\begin{equation}\label{eq:SufficientCondition5}
\prod_{i \neq k} (\sigma_{\max}(\weights_i))^2 = \prod_{i \neq k} (\sigma_{\max} ( \optmlOper ))^{\frac{2}{m}} = (\sigma_{\max} ( \optmlOper ))^{2(1-\frac{1}{m})}.
\end{equation}
Thus, 
\begin{equation}\label{eq:SufficientCondition6}
\lambda_{\max}(\Hess_{\weightsPointVec}) = 2m \times (\sigma_{\max} ( \optmlOper )) ^{2(1-\frac{1}{m})}.
\end{equation}
Therefore \( \weightsPointVec \) is a flattest minimum.

\section{More Details About Our Experiments}\label{appendix:Experiments}
\subsection{Linear Networks}

To sample arbitrary global minima, we started with the canonical solution \eqref{eq:CanonicalSolution}, and multiplied the weight matrices by random matrices from the left and right, such that the the left matrix of one layer cancels out the right matrix of the next (thus keeping the end-to-end function unmodified). Specifically, let $ \{ \bAA_i \}_{i = 1}^{m-1} $ be Gaussian random matrices with i.i.d. entries, distributed $ \mathcal{N}(0,1) $. Then the weights for arbitrary solutions were generated as
\begin{equation}
\weights_m = \UU \SSS_m^{\frac{1}{m}} \bAA_{m-1}, \qquad  \weights_i = \bAA_{i}^{-1} \SSS_i^{\frac{1}{m}}\bAA_{i-1}, \qquad   \weights_1 = \bAA_{1}^{-1} \SSS_1^{\frac{1}{m}} \VV^T.
\end{equation}

To obtain flattest minima, we minimized $\lambda_{\max}( \Hess_{\weightsPointVec} )  $ w.r.t.~the weights, by taking random steps over the manifold of global minima $ \minset $, and greedily progressing towards a flattest solution. In detail, we randomly generated a set of matrices $ \{ \bAA^0_i \}_{i = 1}^{m} $ with i.i.d. normally distributed entries. We then set the initial weights of the network to be $  \weights^{0}_i = \bAA^{0}_i $, for all $ i \neq j $, and 
\begin{equation}
\weights^{0}_j =  \bigg( \prod_{i = j+1 }^{m} \bAA_i \bigg)^{-1} \optmlOper  \bigg( \prod_{i = 1}^{j-1} \bAA_i  \bigg)^{-1} ,
\end{equation}
where $ j $ was a random integer chosen uniformly over $ \{ 1,\ldots,m \} $. Next, we iteratively took small random steps over the manifold of global minima according to the following update rule.
\begin{align}
\weights^{t+1}_m & = \weights^{t}_m \big( \Identity	+\varepsilon_t \bAA^{t}_{m-1} \big),\nonumber\\
\weights^{t+1}_i & = \big(  \Identity + \varepsilon_t \bAA^{t}_{i} \big)^{-1} \weights_i^{t}  \big( \Identity +\varepsilon_t \bAA^{t}_{i-1} \big),\nonumber\\
\weights^{t+1}_1 & = \big(  \Identity + \varepsilon_t \bAA^{t}_{1} \big)^{-1} \weights^{t}_1,
\end{align}
where $ \varepsilon_t $ is the step size at the $ t $th iteration, and $ \{ \bAA^t_i \}_{i = 1}^{m} $ are again random matrices with i.i.d.~normally distributed entries. We continued to the next iteration only if the spectral norm of the Hessian decreased. Otherwise, we generated an additional set of direction matrices $ \{ \bAA^t_i \}_{i = 1}^{m} $ until we got a decrement. We stopped this process when the objective achieved its minimal value of $  2m \times (\sigma_{\max} ( \optmlOper )) ^{2(1-1/m)} $, up to a minor error.

\subsection{Nonlinear Networks}
The table below summarizes the parameters and the results for the methods we used in the nonlinear setting for Fig.~\ref{fig:intermediate_gain_nonlinear}.
\begin{table}[h!]
	\begin{center}
		\begin{tabular}{ |c|c|c| } 
			\hline
			& Method 1 & Method 2 \\
			\hline
			Optimization Algorithm & SGD & Adam \\ 
			Learning rate & $ 1/2$ & $ 3 \times 10^{-4}$ \\
			Other parameters & momentum = 0 & $ \beta_1 = 0.8, \ \beta_2 = 0.99 $ \\
			Batch size & $100$ & $100$ \\
			Train loss & $ 2.69\times10^{-2} \pm  2.53\times10^{-5}  $  & $ 2.74\times10^{-2} \pm  4.03\times10^{-5}   $ \\
			Validation loss & $  2.70\times10^{-2} \pm 1.68\times10^{-4} $  & $ 2.80\times10^{-2} \pm 1.94\times10^{-4} $ \\
			$ \lambda_{\max} $ & $  1.76\pm 9.43\times10^{-3} $  & $ 12.9 \pm 2.2 $ \\
			\hline
		\end{tabular}
		\caption{Summary of the two methods we used to train the network.}
		\label{table:methods}
	\end{center}
\end{table}

\vfill